\documentclass[runningheads]{llncs}

\usepackage{eccv}

\usepackage{eccvabbrv}

\usepackage{graphicx}
\usepackage{booktabs}
\usepackage{amsmath}
\usepackage{amssymb}
\usepackage{mathtools}
\usepackage{algorithm}
\usepackage{algpseudocode}
\usepackage{tcolorbox}
\usepackage{wrapfig}
\usepackage{bbm}
\usepackage{array}
\tcbuselibrary{listings,breakable}
\usepackage{thm-restate}
\usepackage{multirow}
\usepackage{color, colortbl}
\usepackage{pifont}

\definecolor{Gray}{gray}{0.9}
\definecolor{LightBlue}{rgb}{0.85, 0.90, 0.94}
\newcommand{\xmark}{\textcolor{black}{\ding{55}}}
\newcommand{\cg}{\cellcolor{LightBlue}}

\newcommand{\xL}{\vx_{L}}

\newcommand{\xH}{\vx_{H}}

\newcommand{\hxH}{\hat{\vx}_{H}}

\usepackage{amsmath,amsfonts,bm}

\def\eqref#1{equation~\ref{#1}}

\def\1{\bm{1}}

\def\vc{{\bm{c}}}

\def\ve{{\bm{e}}}

\def\vv{{\bm{v}}}
\def\vw{{\bm{w}}}
\def\vx{{\bm{x}}}

\def\vz{{\bm{z}}}

\DeclareMathAlphabet{\mathsfit}{\encodingdefault}{\sfdefault}{m}{sl}
\SetMathAlphabet{\mathsfit}{bold}{\encodingdefault}{\sfdefault}{bx}{n}

\def\gS{{\mathcal{S}}}

\def\sR{{\mathbb{R}}}

\newcommand{\epsilonb}{{\boldsymbol \epsilon}}

\usepackage{capt-of}
\usepackage{diagbox}

\definecolor{C0}{rgb}{0.121569, 0.466667, 0.705882}
\definecolor{C1}{rgb}{1.000000, 0.498039, 0.054902}
\definecolor{C2}{rgb}{0.172549, 0.627451, 0.172549}
\definecolor{C3}{rgb}{0.839216, 0.152941, 0.156863}
\definecolor{C4}{rgb}{0.580392, 0.403922, 0.741176}
\definecolor{C5}{rgb}{0.549020, 0.337255, 0.294118}
\definecolor{C6}{rgb}{0.890196, 0.466667, 0.760784}
\definecolor{C7}{rgb}{0.498039, 0.498039, 0.498039}
\definecolor{C8}{rgb}{0.737255, 0.741176, 0.133333}
\definecolor{C9}{rgb}{0.090196, 0.745098, 0.811765}
\definecolor{trolleygrey}{rgb}{0.5, 0.5, 0.5}

\definecolor{BrickRed}{rgb}{0.6,0,0}
\definecolor{RoyalBlue}{rgb}{0,0,0.8}
\definecolor{Tdgreen}{rgb}{0,0.4,0.7}
\definecolor{pinegreen}{rgb}{0.0, 0.47, 0.44}
\definecolor{cornellred}{rgb}{0.7, 0.11, 0.11}
\definecolor{cadmiumgreen}{rgb}{0.0, 0.42, 0.24}
\definecolor{spirodiscoball}{rgb}{0.06, 0.75, 0.99}
\definecolor{mylightblue}{rgb}{0.85, 0.90, 0.94}
\definecolor{maroon}{cmyk}{0,0.87,0.68,0.32}

\usepackage{pifont}
\definecolor{cfg}{rgb}{0.906, 0.435, 0.318}
\definecolor{cfgpp}{rgb}{0.165, 0.616, 0.561}
\definecolor{cfgnull}{rgb}{0.208, 0.565, 0.953}

\newcommand{\tiledxt}{{\vx_t^{(i)}}}
\newcommand{\tiledxl}{{\vx_L^{(i)}}}
\newcommand{\tiledxh}{{\vx_H^{(i)}}}
\newcommand{\tiledxzero}{{\vx_0^{(i)}}}
\newcommand{\deltati}{{\Delta_{t}^{(i)}(\vc)}}

\newcommand{\idealc}{{\vc^{\ast}}}
\newcommand{\vcglobal}{{\vc_\text{global}}}
\newcommand{\vclocal}{{\vc_\text{local}^{(i)}}}

\newcommand{\idealdeltati}{{\Delta_{t}^{(i)}(\idealc)}}

\newcommand{\deltac}{{\delta_i(\vc)}}

\newcommand{\eb}{\mathbf{e}}

\def\eqref#1{Eq.~(\ref{#1})}

\definecolor{tp}{rgb}{0.906, 0.435, 0.318}
\definecolor{tp2}{rgb}{0.165, 0.616, 0.561}
\definecolor{tp3}{rgb}{0.208, 0.565, 0.953}

\usepackage[accsupp]{axessibility}  

\usepackage[pagebackref,breaklinks,colorlinks,citecolor=eccvblue]{hyperref}

\usepackage{orcidlink}

\begin{document}

\title{Tiled Prompts: Overcoming Prompt Misguidance in Image and Video Super-Resolution}

\titlerunning{Tiled Prompts}

\author{Bryan Sangwoo Kim$^*$ \and
Jonghyun Park$^*$ \and
Jong Chul Ye}

\authorrunning{B. S. Kim et al.}

\institute{KAIST AI \\
\email{\{bryanswkim, jhpark99, jong.ye\}@kaist.ac.kr} \\
{\fontsize{8}{10}\selectfont * Equal contribution}
}

\maketitle

\begin{figure}
    \centering
    \includegraphics[width=.9\linewidth]{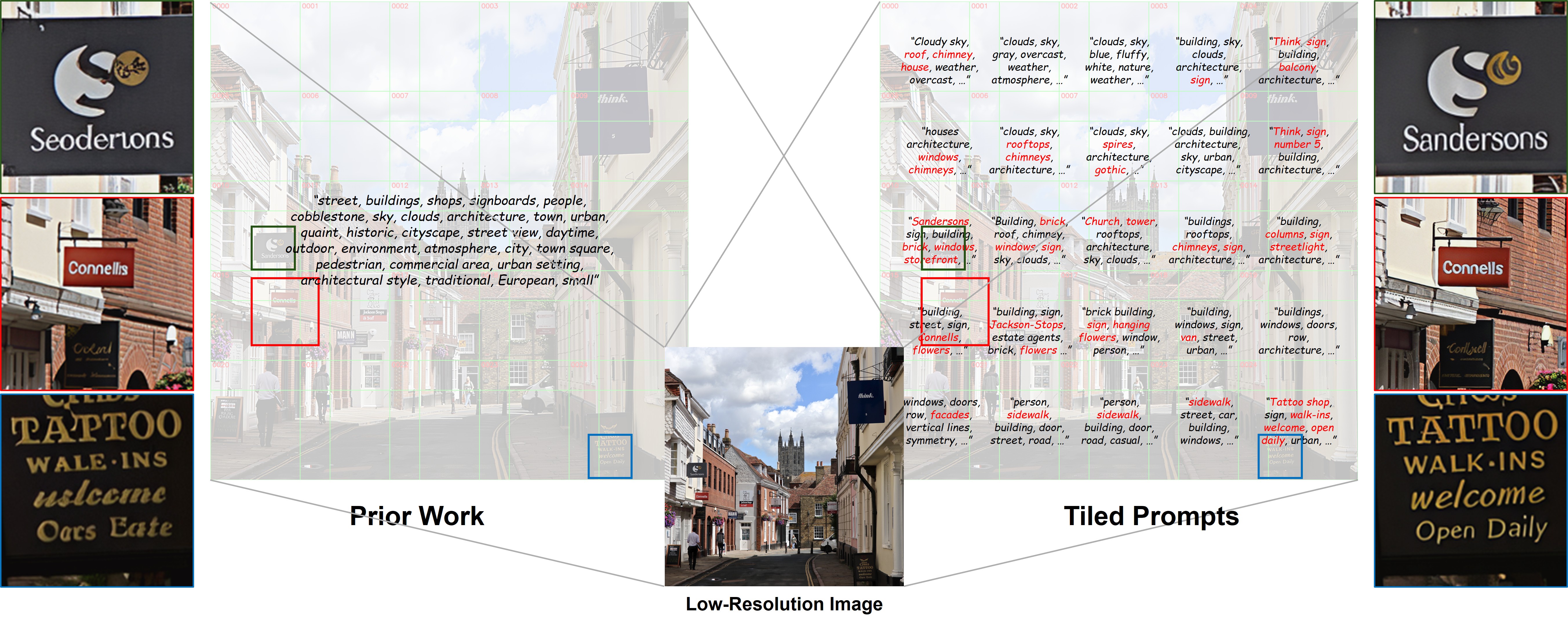}
    \caption{(Left) Relying on a single, global prompt for the latent tiling strategy during super-resolution causes \textit{prompt misguidance}, leading to suboptimal reconstructions. (Right) Using \textit{tiled prompts} solves ambiguity and provides accurate localized guidance needed to reconstruct high-quality details. For example, text on signs are correctly generated only when their corresponding prompts are provided.}
    \label{fig:teaser}
    \vspace{-0.2cm}
\end{figure}

\begin{abstract}
  Text-conditioned diffusion models have advanced image and video super-resolution by using prompts as semantic priors, and modern super-resolution pipelines typically rely on latent tiling to scale to high resolutions. In practice, a single global caption is used with the latent tiling, often causing \textit{prompt misguidance}.
  Specifically, a coarse global prompt often misses localized details (errors of omission) and provides locally irrelevant guidance (errors of commission) which leads to substandard results at the tile level. To solve this, we propose \textit{Tiled Prompts}, a unified framework for image and video super-resolution that generates a tile-specific prompt for each latent tile and performs super-resolution under locally text-conditioned posteriors to resolve prompt misguidance with minimal overhead.
  Our experiments on high resolution real-world images and videos show that tiled prompts bring consistent gains in perceptual quality and fidelity, while reducing hallucinations and tile-level artifacts that can be found in global-prompt baselines.
  Project Page: \url{https://bryanswkim.github.io/tiled-prompts/}.
  
  \keywords{Super-resolution \and Diffusion models \and VLMs}
\end{abstract}

\begin{figure}[t]
    \centering
    \includegraphics[width=\linewidth]{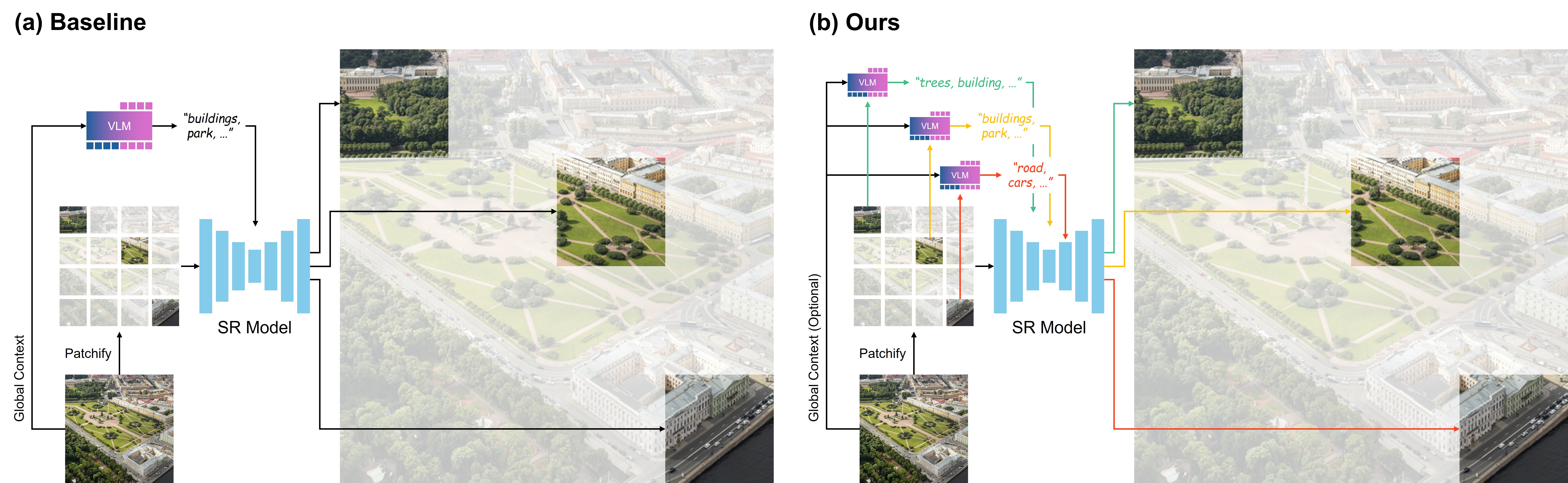}
    \caption{\textbf{(a) Baseline Methods}: Conditioning super-resolution models solely on a single global text prompt demonstrates the problem of prompt misguidance. The global prompt, while broadly describing the image, proves insufficient to constrain the fine-grained super-resolution process. \textbf{(b) Our Method (Tiled Prompts)}: Our framework leverages dense, context-aware \textit{tiled prompts} for each region. This localized textual guidance provides precise semantic anchors, enabling the super-resolution model to produce significantly sharper, more coherent, and perceptually richer details.}
    \label{fig:img-pipeline}
\end{figure}

\section{Introduction}
\label{sec:introduction}

Single Image Super-Resolution (SISR) aims to recover a high-resolution (HR) representation from a low-resolution (LR) input. As a fundamental low-level vision task, it is crucial for a wide range of applications, from enhancing legacy media to improving scientific visualization and medical imaging systems~\cite{betzig2006imaging, oktay2016multi, ravishankar2010mr, wang2022comprehensive}. However, SISR is an ill-posed inverse problem; for any given LR input, there exists a multitude of plausible HR solutions.
Modern deep learning approaches typically address this by learning the posterior probability distribution
\begin{equation}
    p(\xH \mid \xL)
    \label{eq:standardSR}
\end{equation}
from large-scale datasets, where the LR input $\xL$ is mapped to its HR reconstruction $\xH$.
When properly trained, these models learn to effectively constrain the solution space, ensuring that the generated outputs are perceptually realistic.

Recent methods for super-resolution have significantly advanced this field by incorporating pre-trained generative models (\textit{e.g.}, diffusion models), exploiting their rich generative priors for the super-resolution task~\cite{wang2024exploiting, wu2024seesr, yu2024scaling, wu2024one}. A key advantage of these generative frameworks is their ability to leverage textual conditioning. Since modern generative backbones are trained on diverse image-text pairs~\cite{rombach2022high, esser2024scaling}, they naturally allow for text-guided generation. This effectively reformulates \eqref{eq:standardSR} by conditioning the posterior probability distribution on additional textual guidance as follows:
\begin{equation}
    p(\xH \mid \xL, \vc_\text{global})
    \label{eq:textSR}
\end{equation}
where $\vc_\text{global}$ is a textual condition for providing global context about the input image.
These textual priors are critical in creating high-frequency details and serving as powerful semantic anchors that guide the reconstruction toward plausible outcomes~\cite{wu2024one, kim2025chain}.

\begin{wrapfigure}{r}{0.5\textwidth}
    \includegraphics[width=0.45\textwidth]{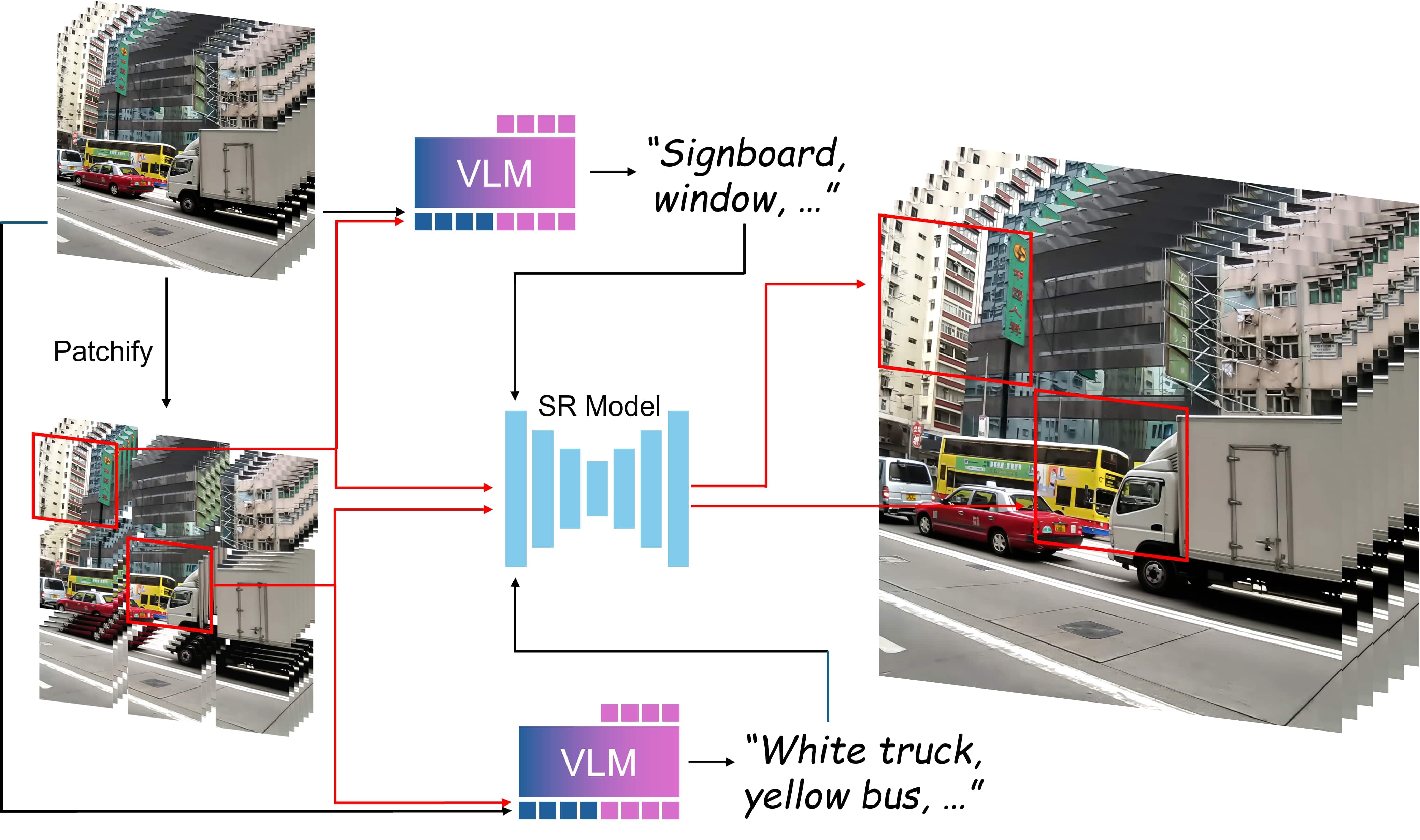}
    \caption{Our \textbf{Tiled Prompts} framework for VSR divides the low-resolution video into a grid of spatio-temporal blocks (or volumes), where each block is tiled both spatially and temporally. A VLM then analyzes the local video content and generates a detailed text prompt to guide the reconstruction of the specific block.}
    \label{fig:video-pipeline}
    \vspace{-1cm}
\end{wrapfigure}

This trend toward text-conditioned super-resolution extends naturally to Video Super-Resolution (VSR), currently an active area of research.
Text conditioning in VSR can be expressed by simply replacing the image pair $\{ \xH, \xL \}$ with a temporally dependent sequence pair $\{ \xH^{1:N}, \xL^{1:N} \}$, which enables the joint super-resolution of the $N$ frames comprising a video:
\begin{equation}
    p(\xH^{1:N} \mid \xL^{1:N}, \vc_\text{global})
    \label{eq:textSR-video}
\end{equation}
Analogous to SISR, text-conditioning for VSR is important for reconstruction of high-frequency details.

However, as interest in increasingly high-resolution imagery continues to grow, standard text-conditioning paradigms face a critical limitation. Modern super-resolution methods typically rely on a \textit{latent tiling strategy}~\cite{jimenez2023mixture, bartal2023multidiffusionfusingdiffusionpaths} of processing small, overlapping tiles rather than full-sized latents, to mitigate the roughly quadratic growth in VRAM usage as resolution scales.
In this tiled regime, we encounter the problem of \textit{prompt misguidance}: textual guidance derived from a single global caption $\vc_\text{global}$ often lacks the specific localized details needed to reconstruct individual tiles faithfully.
At extreme resolutions (\textit{e.g.}, 4K super-resolution), the number of tiles can easily reach on the order of hundreds.
This semantic misalignment between the global prompt and local visual content inhibits the fine-grained textual guidance needed for plausible SR, a limitation that arises commonly in both SISR and VSR.

In this work, we propose \textit{Tiled Prompts}: a unified framework to overcome prompt misguidance in both image and video super-resolution. Rather than relying on a single, often irrelevant global prompt $\vc_\text{global}$, our approach generates a dense, descriptive prompt $\vc_\text{local}^{(i)}$ specific to each \textit{i-th} tile, used for the super-resolution from $\tiledxl$ to $\tiledxh$:
\footnote{For simplicity, we unify both image and video cases with the notations $\xL\in\sR^{T\times H\times W\times C}$ and $\xH\in\sR^{T\times rH\times rW\times C}$ for LR input and HR output respectively ($T=1$ corresponds to SISR).}
\begin{equation}
    p(\tiledxh \mid \tiledxl, \vclocal)
    \label{eq:textSR-local}
\end{equation}
We give theoretical insight into how these local prompts reduce ambiguity and hallucination, thereby resolving prompt misguidance. Specifically, the proposed framework solves \textit{errors of omission} by supplementing prompts that can be overlooked when generating a single global prompt (\textit{e.g.}, the red prompts in Fig.~\ref{fig:teaser}), and solves \textit{errors of commission} by eliminating highly irrelevant prompts that can induce off-target semantics.

The proposed framework is highly effective for alleviating prompt misguidance in both SISR and VSR. Though VSR shares the same fundamental goal as SISR, VSR demands the extra challenge of enforcing consistency along an additional temporal axis. This added dimension significantly exacerbates the misguidance problem, making it even more inadequate for a single global prompt to be used to explain an entire video. We demonstrate that tiled prompts capture dynamic content, scene changes, and localized details that are otherwise overlooked, all the while introducing minimal computational overhead.
%
Our contributions can be summarized as follows:
\begin{enumerate}
    \item We define the \textit{prompt misguidance} problem found when using latent tiling, and analyze its negative impacts for super-resolution. In particular, we introduce two types of semantic deviation observed and categorize them as \textit{errors of commission} and \textit{errors of omission}.
    \item We introduce a unified framework to mitigate prompt misguidance that utilizes \textit{tiled prompts} for localized textual guidance, and formally show that misguidance can be minimized by lowering an information-theoretic bound.
    \item Experimentally, we show that tiled prompts consistently improve performance in both SISR and VSR settings, even with negligible overhead.
\end{enumerate}

\section{Related Work}
\label{sec:related-work}

\subsection{Text-guided Image Super-Resolution}

State-of-the-art models incorporate pre-trained generative models, and fine-tune their generative priors for super-resolution. During this process, textual priors are used to guide high frequency details of the super-resolution process. 
StableSR~\cite{wang2024exploiting} employs a diffusion prior without any prompt, while SUPIR~\cite{yu2024scaling} trains the model with detailed text prompts to enrich semantic guidance.
SeeSR~\cite{wu2024seesr} introduces a Degradation-Aware Prompt Extractor (DAPE) that injects semantically rich prompts to encourage finer detail generation. 
OSEDiff~\cite{wu2024one} and PiSA-SR~\cite{sun2025pixel} further accelerate the SR process into a one-step denoising framework, while using text guidance for fine details.
More recently, TeReDiff~\cite{min2025text} proposes Text-Aware Image Restoration (TAIR), which leverages explicit textual cues extracted from the image to faithfully recover text that is otherwise ambiguous in low-quality inputs.

As the role of textual guidance becomes increasingly prominent, the super-resolution community has begun to incorporate Vision-Language Models (VLMs) to leverage their visual understanding capabilities. Recent baselines such as SUPIR~\cite{yu2024scaling} and DiT4SR~\cite{duan2025dit4sr} use LLAVA~\cite{liu2023visual} as a prompt extractor. Chain-of-Zoom~\cite{kim2025chain} studies extreme super-resolution with VLMs, showing that scale autoregression, coupled with carefully structured text guidance, can progressively improve SR quality through iterative magnification.
However, to the best of our knowledge, prior work has not explored using VLMs to provide effective guidance across the full spatial extent of high-resolution images during tiled SR inference. In this work, we address prompt misguidance by incorporating tile-level prompting via VLMs to supply localized guidance.

\subsection{Video Super-Resolution}
Recent works on Video Super-Resolution (VSR) aim to fully exploit the generative capability of diffusion-based models by incorporating text-to-image (T2I) and text-to-video (T2V) priors.
%
%
Upscale-A-Video~\cite{zhou2023upscale} and MGLD-VSR~\cite{yang2024motion} fine-tune pretrained T2I models by leveraging optical flow between adjacent frames.
While effective to some extent, such fine-tuning compromises the diffusion model’s inherent image generation fidelity.
DLoRAL~\cite{sun2025one} introduces a dual-LoRA framework to preserve generative quality while improving temporal stability, but temporal inconsistency remains a problem for T2I-based VSR approaches.

Meanwhile, leveraging T2V priors for VSR helps alleviate temporal inconsistency, and various methods have been explored to improve per-frame fidelity. STAR~\cite{xie2025star} augments T2V backbones (\textit{e.g.}, I2VGen-XL~\cite{2023i2vgenxl}, CogVideoX~\cite{yang2025cogvideox}) with a lightweight local-detail enhancement module to better preserve fine structures, along with a frequency-aware objective for improved fidelity. SeedVR~\cite{wang2025seedvr} proposes a scalable diffusion-transformer backbone for generic video restoration via (shifted) window attention, and SeedVR2~\cite{wang2025seedvr2} improves this through diffusion adversarial post-training. Stream-DiffVSR~\cite{shiu2025streamdiffvsr} introduces a causally conditioned diffusion framework for low-latency online VSR that operates only on past frames.
However, despite that these models are dependent on textual guidance for accurate reconstruction, none of these prior works fully validate the correctness of the prompts that are used.
Thus, in this work we reveal that these models are actually prone to prompt misguidance when used for high resolutions.

\section{Preliminaries}
\label{sec:preliminaries}

\subsection{Latent Tiling during Inference of SR Models}
\label{sec:latent-tiling}
When input LR data exceeds the optimal processing size of the SR model, current methods split the input into a grid of tiles and perform SR at the tile-level to satisfy memory and compute constraints.
This \textit{latent tiling} strategy is not unique to super-resolution; it is widely used in high-resolution generative modeling (\textit{e.g.}, MultiDiffusion~\cite{bartal2023multidiffusionfusingdiffusionpaths}, Mixture of Diffusers~\cite{jimenez2023mixture}), where overlapping tiles are processed independently and then merged to produce a full-resolution sample.
For latent tiling size of $t\times k_1\times k_2$, $\xL$ is split into a grid of tiles,
\begin{equation}
    \left\{ \xL^{(i)} \right\}_{i=1}^N, \quad \quad \xL^{(i)} \in \sR^{t\times k_1\times k_2\times C}
    \label{eq:gridLR}
\end{equation}
where $\xL^{(i)}$ corresponds to an HR region $R^{(i)} \subset \{1,\cdots ,T\}\times\{1,\cdots,rH\} \times \{1,\cdots,rW\}$ of size $t\times (rk_1)\times(rk_2)$ at the scaled location.
In practice, tiling is performed in the model's latent space rather than directly in pixel space, but we retain the notation in \eqref{eq:gridLR} to denote these latent tiles for simplicity.

Given a global text condition $\vc_\text{global}$, the text-conditioned SR model trained on the conditional probability distribution of \eqref{eq:textSR} operates \textit{per-tile} as follows:
\begin{equation}
    \hxH^{(i)} \sim f_\theta(\xH^{(i)} \mid \xL^{(i)}, \vc_\text{global})
    \label{eq:patchSR}
\end{equation}
where $f_\theta$ denotes the SR model with parameters $\theta$.
The resulting full HR output is the direct placement of predicted tiles onto their HR grid cells:
\begin{equation}
    \hxH(t,h,w)=\sum_{i=1}^N{\mathbbm{1}_{\{ (t,h,w) \in R_i \}} \hxH^{(i)}(\phi_i(t,h,w)) }
    \label{eq:gridHR}
\end{equation}
where $\mathbbm{1}_{\{ (t,h,w) \in R_i \}}$ is the indicator for region $R_i$, and $\phi_i$ maps global HR coordinates $(t,h,w)\in R_i$ to local coordinates within tile $i$ of size $t\times(rk_1)\times(rk_2)$.
Equivalently, the global posterior is approximated as a product of disjoint tile-wise conditionals:
\begin{equation}
    p(\xH \mid \xL, \vc_\text{global}) \approx \prod_{i=1}^N{p_i\left( \xH^{(i)} \mid \xL^{(i)}, \vc_\text{global} \right)}
    \label{eq:global-posterior}
\end{equation}
where $\hxH$ is obtained by the aggregation rule in Eq.~\ref{eq:gridHR}.

In practice, direct aggregation often introduces boundary discontinuities, so a standard practice is to overlap tiles with Gaussian blending~\cite{liang2021swinir}. Let $w_i(h,w) \geq 0$ be a Gaussian weighting window supported on $R_i$ (and zero elsewhere). The aggregated HR estimate is then given by:
\begin{equation}
    \hxH(h,w) = \frac{\sum_{i=1}^N {w_i(h,w) \hxH^{(i)}(\phi_i(h,w))}}{\sum_{i=1}^N {w_i(h,w)}}
\end{equation}

\subsection{Unified Forward Process and Text Guidance}
\label{sec:unified_forward}
To encapsulate a wide range of continuous-time generative models including diffusion models~\cite{ho2020denoising, song2020score, song2020denoising}, and flow-based models~\cite{lipman2022flow, liu2023flow} within a unified framework, we define the forward process via its marginal distribution:
\begin{equation}
    p_t(\tiledxt | \tiledxzero) = \mathcal{N}(\tiledxt; \alpha_t \tiledxzero, \sigma_t^2 \mathbf{I})
\end{equation}
where $\alpha_t$ and $\sigma_t$ define the specific noise schedule or flow trajectory. For general diffusion models, this encompasses Variance Preserving (VP) ($\alpha_t^2 + \sigma_t^2 = 1$), Variance Exploding (VE) ($\alpha_t = 1$), and generalized formulations like EDM~\cite{karras2022elucidating}. For linear flow-based models, the trajectory follows $\alpha_t = 1-t$ and $\sigma_t = t$.

For the super-resolution task, conditional diffusion models are often modeled as noise prediction networks that take both the LR tile $\tiledxl$ and the text prompt $\vc$ as conditions. For such cases, Classifier-Free Guidance (CFG)~\cite{ho2022classifier} is utilized as a standard inference mechanism to increase adherence to the text condition without modifying model parameters:
\begin{equation}
    \begin{aligned}
    \tilde{\eb}_\theta(\tiledxt | \tiledxl, \vc)
    &= \eb_\theta(\tiledxt | \tiledxl, \varnothing)
    + s [ \eb_\theta(\tiledxt | \tiledxl, \vc) - \eb_\theta(\tiledxt | \tiledxl, \varnothing) ] \\
    &= \eb_\theta(\tiledxt | \tiledxl, \varnothing)
    + s\deltati,
    \end{aligned}
\label{eq:cfg}
\end{equation}
where $\eb$ is the model prediction (\textit{e.g.}, noise $\epsilonb_\theta$ or velocity $\vv_\theta$), $s>1$ is the guidance scale, and $\deltati \coloneqq \eb_\theta(\tiledxt | \tiledxl, \vc) - \eb_\theta(\tiledxt | \tiledxl, \varnothing) $ is the guidance direction.
The guidance direction $\deltati$ steers the generative trajectory according to the semantics implied by the given text condition $\vc$.

\clearpage
\section{Tiled Prompts}
\label{sec:methodology}

\subsection{Prompt Misguidance}
\label{sec:methodology-misguidance}

Existing diffusion-based SR methods generally extract a single global prompt $\vc_\text{global}$ from the full LR input $\xL\in\sR^{T\times H\times W\times C}$ and reuse it for all local tiles $\xL^{(i)}$:
\begin{equation}
    \label{eqn:general_SR}
    \hxH^{(i)} \sim f_\theta(\xH^{(i)} \mid \xL^{(i)}, \vcglobal), \quad \vcglobal\sim p(\vc|\xL).
\end{equation}
While $\vcglobal$ captures the overall scene context, it acts as a coarse constraint when applied to local tiles. Thus, tile-specific attributes are often omitted in the global prompt, while irrelevant details that do not appear in the \textit{i-th} tile are included, leading to \textit{prompt misguidance}.

We formally define prompt misguidance by first assuming the existence of an ideal textual condition $\idealc\sim p^*(\idealc|\tiledxl)$ that perfectly describes the local visual evidence for the $i$-th tile. We define the \textit{prompt misguidance} $\deltac$ as follows:
\begin{align}
    \deltac &\coloneqq \deltati - \idealdeltati, \label{eqn:deltac_def} \\
    \tilde{\eb}_\theta(\tiledxt | \tiledxl, \vc)
    &= \eb_\theta(\tiledxt | \tiledxl, \varnothing)
    + s\idealdeltati + s\deltac, \label{eqn:deltac_cfg}
\end{align}
where \eqref{eqn:deltac_cfg} demonstrates how $s\deltac$ shifts the generative trajectory from the optimal direction. We can explicitly represent the term $\deltac$ via the score function, across different model parameterizations.

\begin{restatable}{lemma}{lemmamisguidance}
    \label{lemma:misguidance}
    Let the forward process be defined by the generalized Gaussian marginal $p_t(\tiledxt | \tiledxzero) = \mathcal{N}(\tiledxt; \alpha_t \tiledxzero, \sigma_t^2 \mathbf{I})$. Then, the prompt misguidance $\deltac$ defined in \eqref{eqn:deltac_def} can be universally represented as:
    \begin{equation}
        \deltac = w(t)\left(\nabla_\tiledxt\log{p_t(\tiledxt |\tiledxl,\idealc)} - \nabla_\tiledxt\log{p_t(\tiledxt |  \tiledxl, \vc)} \right)
    \end{equation}
    where the weighting function $w(t)$ varies depending on the prediction target:
    \begin{equation}
        w(t) = 
        \begin{cases}
            \frac{\sigma_t^2}{\alpha_t}, & \text{for } \vx_0\text{-prediction} \\
            \sigma_t, & \text{for } \epsilonb\text{-prediction} \\
            \frac{\sigma_t}{\alpha_t}, & \text{for } \vv\text{-prediction and flow-based models}
        \end{cases}
    \end{equation}
\end{restatable}

Empirically, two types of semantic deviation are observed in super-resolution settings with latent tiling: (i) errors of commission and (ii) errors of omission. Errors of commission are when $\vcglobal$ contains semantic tokens describing other regions of the image (\textit{e.g.}, ``sky'' for a tile showing pavement). The score gradients maximizing the likelihood of these extraneous concepts push the reverse diffusion trajectory toward completely different semantics. Errors of omission are when the global prompt being coarse omits critical local high-frequency details. Because $\tiledxl$ alone is highly ill-posed, the absence of specific text prompts can leave the conditional score direction ambiguous, failing to provide accurate guidance to the correct sharp posterior.

\clearpage
\subsection{Information-Theoretic Bound on Prompt Misguidance}
\label{sec:methodology-mi-connection}

To systematically mitigate the prompt misguidance $\deltac$, we must establish its relationship with the information-theoretic properties of the conditioning text. We evaluate the discrepancy between using an ideal local prompt $\idealc$ and any given text condition $\vc$ by defining the gap in mutual information $\Delta I$:
\begin{equation}
    \label{eqn:deltaIdef}
    \Delta I \coloneqq I(\xH^{(i)} ; \idealc \mid \xL^{(i)}) - I(\xH^{(i)} ; \vc \mid \xL^{(i)}).
\end{equation}
By definition, the ideal local prompt $\idealc$ completely captures all localized visual semantics. Consequently, conditioned on $\tiledxl$ and $\idealc$, the high-resolution tile $\tiledxh$ is conditionally independent of any other generic text condition $\vc$:
\begin{equation}
    \label{eqn:cond_indep}
    p(\tiledxh \mid \tiledxl, \idealc, \vc) \approx p(\tiledxh \mid \tiledxl, \idealc).
\end{equation}
According to the unified framework of score-based generative models~\cite{song2020score}, any continuous-time Gaussian forward process $p_t(\tiledxt | \tiledxzero)$ including the marginal distributions of deterministic flow-based models can be characterized by an equivalent SDE of the form $d\tiledxt = f(t)\tiledxt dt + g(t)dw_t$. By leveraging the path measures of these associated SDEs, 
we can directly link the mutual information gap $\Delta I$ to the KL divergence and bound it over the entire trajectory. 

\begin{restatable}{proposition}{propmutual}
\label{prop:mutual}
    For any text condition $\vc$, the mutual information gap $\Delta I$
    is equivalent to the KL divergence between the ideal and $\vc$-guided posteriors:
    \begin{align}
        \label{eqn:DeltaIKL}
        \Delta I = \mathbb{E}_{\idealc,\vc|\tiledxl} \left[D_{KL}( p(\xH^{(i)} \mid \xL^{(i)}, \idealc) \parallel p(\xH^{(i)} \mid \xL^{(i)}, \vc) )\right].
    \end{align}
    Moreover, it lower-bounds the accumulated prompt misguidance:
    \begin{equation}
        \Delta I \leq \frac{1}{2} \int_0^T \lambda(t) \mathbb{E}_{\idealc, \vc|\tiledxl}\mathbb{E}_{\tiledxt\sim p_t(\tiledxt|\tiledxl,\idealc)} \left[ \| \deltac(t) \|^2 \right] dt,
    \end{equation}
    with $\lambda(t) = \frac{g(t)^2}{w(t)^2}$, where $g(t)$ is the diffusion coefficient of the associated forward SDE, and $w(t)$ is the parameterized weighting function defined in Lemma~\ref{lemma:misguidance}.
\end{restatable}


\subsection{Minimizing Misguidance via Tiled Prompts}
\label{sec:methodology-tiled-prompts}

Proposition~\ref{prop:mutual} reveals a critical bottleneck in standard tiled SR: any loss of mutual information ($\Delta I > 0$) caused by suboptimal $\vc$ artificially raises the lower bound, mathematically forcing the prompt misguidance integral to be strictly greater than zero. To overcome this limitation, we design our proposed Tiled Prompts framework to strictly \textit{reduce} the mutual information gap $\Delta I$ by introducing local prompts $\vclocal$ of much higher relevance compared to $\vcglobal$.

When reconstructing the $i$-th tile with $\vcglobal$, the generative SR model's attention is forced to distribute probability mass across the limited concepts available in $\vcglobal$, or default to an ambiguous prior if relevant local details are missing.
Using $\vclocal$ instead resolves such errors of commission and omission, effectively reducing $\Delta I$ to minimize discrepancy from ideal guidance.

\begin{restatable}{proposition}{proptiled}
\label{prop:tiled}
    Let the global prompt $\vcglobal$ be modeled as a semantic mixture of its present concepts and omitted conditions.
    Then, the mutual information gap observed when using the specific tiled prompt $\vc_\text{local}^{(i)}\sim p(\vc|\tiledxl)$ is less than or equal to the mutual information gap observed when using the global prompt:
    \begin{equation}
        \Delta I_\ell \leq \Delta I_{\mathrm{g}},
    \end{equation}
    where $\Delta I_\ell$ and $\Delta I_{\mathrm{g}}$ denote the mutual information gaps when conditioning on $\vclocal$ and $\vcglobal$, respectively.
\end{restatable}
By extracting and utilizing dense, descriptive prompts $\vc_\text{local}^{(i)}$ specific to each $i$-th tile, our Tiled Prompts framework explicitly minimizes $\Delta I$ by preventing both errors of commission and omission. According to Proposition~\ref{prop:mutual}, the reduction $\Delta I_\ell \leq \Delta I_{\mathrm{g}}$ is essential in that it drops the theoretical floor of the prompt misguidance vector $\| \deltac \|^2$, allowing the generative trajectory to successfully converge to accurate high-frequency details.

\subsection{Generation of Tiled Prompts}

\begin{wrapfigure}{r}{0.55\textwidth}
\vspace{-1.6cm}
\begin{minipage}{0.55\textwidth}
\begin{algorithm}[H]
    \caption{Inference with Tiled Prompts}
    \label{alg:tiledprompts}
    \begin{algorithmic}[1]
        \Require {LR Input $\xL$, SR Model $f_\theta$, VLM $Y_\text{VLM}$, crop size $(t, k_1, k_2)$, scale $r$,  timesteps $\{\tau_m\}_{m=1}^T$, scheduler $\gS$.} 
        \State {Compute tiling coordinates $\{R_i\}_{i=1}^N$}
        \For {$i = 1 \rightarrow N$}
            \State {$\xL^{(i)} \leftarrow \text{Crop}(\xL, R_i)$}
            \State {$\textcolor{tp}{\vc_\text{local}^{(i)}} \sim Y_\text{VLM}(\xL^{(i)};\eta_i)$}
        \EndFor
        \State {Compute Gaussian window $w_i$ for region $R_i$}
        \State {Initialize $\vz^{(T)} \leftarrow \text{InitNoise}()$}
        \For {$m = T \rightarrow 1$}
            \For {$i = 1 \rightarrow N$}
                \State {$\vz^{(i,m)} \leftarrow \text{Crop}(\vz^{(m)}, R_i)$}
                \State {$\hat{\ve}^{(i)} \leftarrow f_\theta(\vz^{(i,m)}, \xL^{(i)}, \textcolor{tp}{\vc_\text{local}^{(i)}}, \tau_m)$}
                \State {$\mathbf{E}(R_i) \leftarrow \mathbf{E}(R_i) + w_i \cdot \hat{\ve}^{(i)}$}
                \State {$\mathbf{W}(R_i) \leftarrow \mathbf{W}(R_i) + w_i$}
            \EndFor

            \State {$\hat{\ve} \leftarrow \mathbf{E} \oslash \mathbf{W}$}
            \State {$\vz^{(m-1)} \leftarrow \gS( \vz^{(m)}, \hat{\ve}, \tau_m )$}
        \EndFor
        \State {$\xH \leftarrow \text{Decode}(\vz^{(0)})$}
        \State {{\bfseries return} $\xH$}
    \end{algorithmic}
\end{algorithm}
\end{minipage}
\vspace{-0.8cm}
\end{wrapfigure}

\subsubsection{Image Super-Resolution.}

To address prompt misguidance, our framework replaces the single global prompt with \textit{tile-specific} textual conditions. Specifically, instead of broadcasting a single caption $\vc_{\text{global}}$ to all tiles as in Fig.~\ref{fig:img-pipeline}(a), we extract local prompts using a VLM conditioned on each local tile as in Fig.~\ref{fig:img-pipeline}(b). This allows for dense, localized guidance that is better aligned with the visual evidence within each tile.

For each LR tile $\tiledxl$, we generate a corresponding local prompt by leveraging a VLM as a prompt extractor,
\begin{equation}
    \vclocal := Y_{\text{VLM}}(\tiledxl;\eta_i),
    \label{eq:clocal-image}
\end{equation}
where $\eta_i$ is sampling noise.
Collectively, the set of local prompts $\{\vclocal\}_{i=1}^N$ effectively
mitigates misguidance by providing localized, tile-specific semantics that are typically absent or misleading in a single global caption. Tiled prompts reduce the misguidance $\deltac$ by constructing $\vc=\vc_{\text{local}}^{(i)}$ to be semantically consistent with the tile, thereby improving fidelity and reducing artifacts at high resolutions.

\begin{figure}
    \centering
    \includegraphics[width=\linewidth]{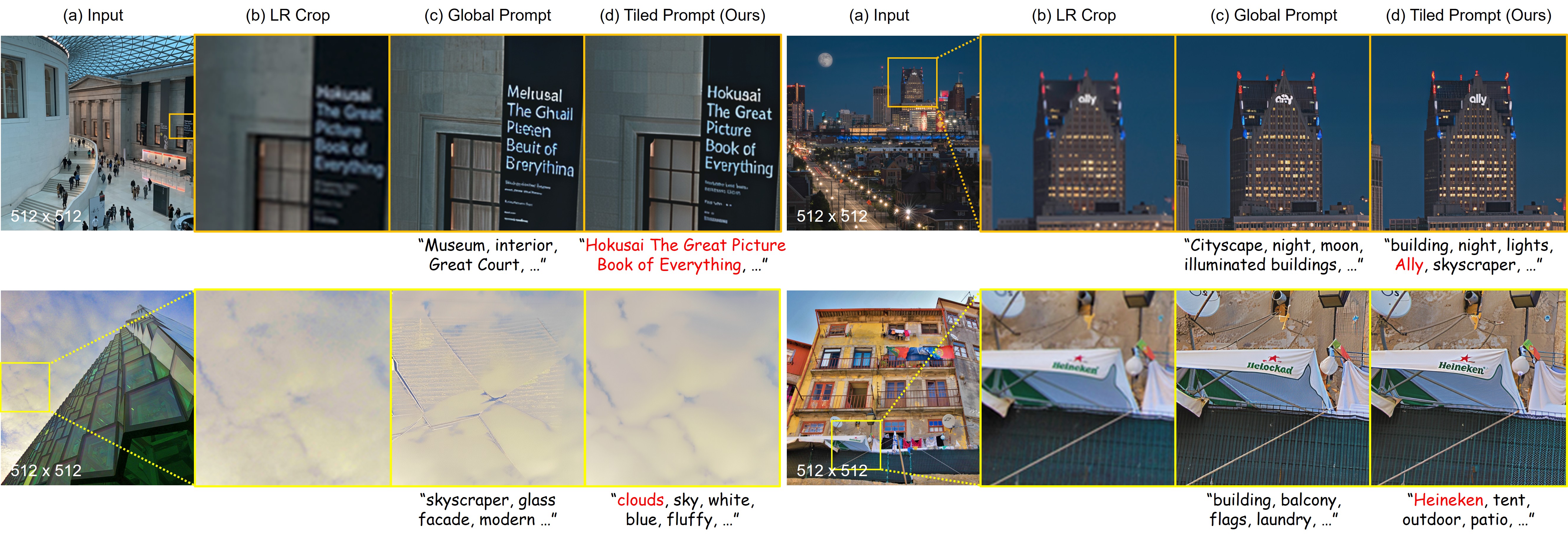}
    \caption{Qualitative results for image super-resolution. \textbf{(a,b) Input}: The low-resolution input and a cropped tile to be upsampled. \textbf{(c) SR with Global Prompt}: The baseline result using only a single, coarse global prompt. As the text prompt does not provide sufficient guidance, super-resolution results are inaccurate. \textbf{(d) SR with Tiled Prompt (Ours)}: Our method uses dense, localized textual guidance to generate accurate and semantically plausible high-frequency details consistent with the context.}
    \label{fig:qualitative-image}
\end{figure}

\subsubsection{Video Super-Resolution.}

Prompt misguidance is exacerbated in Video Super Resolution (VSR) pipelines, where an entire LR sequence $\xL\in\mathbb{R}^{T\times H\times W\times C}$ is conditioned on a single static global prompt $\vc_{\text{global}}$.
This leads to \textit{dual} misguidance during latent tiling: \textit{spatial misguidance}, where $\vc_{\text{global}}$ omits spatial details as for the image case, and \textit{temporal misguidance}, where $\vc_{\text{global}}$ lacks the temporal granularity needed to describe local dynamics (\textit{e.g.}, motion, evolution).

Directly using \eqref{eq:clocal-image} would extract a prompt from each local spatio-temporal block alone.
However, in practice, VLMs struggle to infer accurate motion level descriptions from a small cropped window alone for videos.
In accordance with the increased misguidance, we provide the prompt-extraction VLM with additional global context from the full LR sequence $\xL$,
\begin{equation}
    \vclocal := Y_{\text{VLM}}(\tiledxl,\xL;\eta_i),
    \label{eq:video-prompt-extraction}
\end{equation}
where $\eta_i$ denotes sampling noise.
Although the global input $\xL$ is now shared across tiles, the output $\vclocal$ remains tile-specific to the local evidence in $\xL^{(i)}$ through explicit instructions provided by system and user prompts. These instructions assist the VLM to account for the additional global context and fully leverage the reasoning capability of the VLM, enabling $\vclocal$ to encode localized dynamics  (\textit{e.g.}, ``spinning wheels'' instead of only ``car'') and mitigate both spatial and temporal misguidance.

\begin{table}
    \centering
    \caption{Quantitative comparison on image quality metrics and image-text alignment metrics. Best results are in \textbf{bold}, second-best results are \underline{underlined}.}
    \resizebox{0.8\textwidth}{!}{
    \begin{tabular}{ll|cccc|ccc}
    \toprule
    \multicolumn{2}{l}{} & \multicolumn{4}{c}{\textbf{Image Quality}} & \multicolumn{3}{c}{\textbf{Image-Text Alignment}} \\
    Dataset & Prompt Type & NIQE$\downarrow$ & MUSIQ$\uparrow$ & MANIQA$\uparrow$ & CLIPIQA$\uparrow$ & CLIP Score$\uparrow$ & ImageReward$\uparrow$ & HPSv2$\uparrow$ \\
    
    \midrule
    \multirow{4}{*}{LSDIR1K}
    & \multicolumn{1}{c|}{\xmark}
    & 3.4537 & 62.2188 & 0.6126 & 0.6346 & \xmark & \xmark & \xmark \\
    & Global (Baseline)
    & 2.9427 & 63.8677 & \underline{0.6373} & 0.6886 & 25.3348 & -1.5901 & 0.1589 \\
    & Global $+$ Local
    & \underline{2.9418} & \textbf{64.0749} & \textbf{0.6379} & \textbf{0.6932} & \underline{25.7925} & \underline{-1.4775} & \underline{0.1719} \\
    & \cg Local
    & \cg \textbf{2.9040} & \cg \underline{63.9731} & \cg 0.6350 & \cg \underline{0.6917} & \cg \textbf{27.2274} & \cg \textbf{-0.6771} & \cg \textbf{0.2011} \\

    \midrule
    \multirow{4}{*}{URBAN100}
    & \multicolumn{1}{c|}{\xmark}
    & 4.2366 & \textbf{56.1273} & 0.6343 & 0.6297 & \xmark & \xmark & \xmark \\
    & Global (Baseline)
    & 3.6156 & 53.1372 & \underline{0.6668} & 0.6715 & 25.9807 & -1.1688 & 0.1780 \\
    & Global $+$ Local
    & \underline{3.5724} & 53.8299 & \textbf{0.6681} & \textbf{0.6784} & \underline{26.4686} & \underline{-1.0504} & \underline{0.1895} \\
    & \cg Local
    & \cg \textbf{3.5001} & \cg \underline{54.9203} & \cg 0.6618 & \cg \underline{0.6780} & \cg \textbf{27.4044} & \cg \textbf{-0.5193} & \cg \textbf{0.2092} \\

    \midrule
    \multirow{4}{*}{OST300}
    & \multicolumn{1}{c|}{\xmark}
    & 3.3776 & 62.7091 & 0.6164 & 0.6183 & \xmark & \xmark & \xmark \\
    & Global (Baseline)
    & 2.9547 & 66.3813 & \underline{0.6536} & 0.6799 & 25.1512 & -1.6638 & 0.1418 \\
    & Global $+$ Local
    & \underline{2.9455} & \textbf{66.8685} & \textbf{0.6555} & \underline{0.6875} & \underline{25.4540} & \underline{-1.6344} & \underline{0.1522} \\
    & \cg Local
    & \cg \textbf{2.9007} & \cg \underline{66.7233} & \cg 0.6518 & \cg \textbf{0.6943} & \cg \textbf{26.8594} & \cg \textbf{-0.7881} & \cg \textbf{0.1838} \\
    
    \bottomrule
    \end{tabular}
    }
    \label{tab:quant-image}
\end{table}

\section{Experiments}
\label{sec:experiments}

\subsection{Experimental Settings}

\subsubsection{Image Experiments.}
We use DiT4SR~\cite{duan2025dit4sr} as our base SR model, a recent model that uses a diffusion transformer as its generative backbone. Global and local (tiled) prompts are extracted with Qwen2.5-VL-7B-Instruct~\cite{qwen2.5-VL}.

We evaluate our framework on real-world image datasets Urban100~\cite{huang2015single}, OST300~\cite{wang2018recovering}, and the first 1K images from LSDIR~\cite{li2023lsdir} which we term LSDIR1K. We do not use low-resolution datasets such as DrealSR~\cite{wei2020component}, RealSR~\cite{cai2019toward}, RealLR200~\cite{wu2024seesr}, or RealLQ250~\cite{ai2024dreamclear}, as we find their low resolutions inappropriate for evaluating latent tiling performance. Each image is resized and cropped to a resolution of 512$\times$512, producing a resolution of 2048$\times$2048 after 4$\times$ magnification. For tiled inference, each LR latent is divided into tiles of size 64$\times$64 with overlap of 16, resulting in 25 tiles total.
Such settings require us to use ground-truth images of at least 2048$\times$2048 resolution for evaluating reference-based metrics, which is unavailable for most datasets. Thus, for reference-based evaluation we identify all images larger than 2048$\times$2048 from the entire LSDIR dataset and refer to the resulting 153 images as LSDIR2048.
Reference-based evaluation is performed for the high-resolution LSDIR2048 and SEPE8K~\cite{al2023sepe} datasets, of which ground-truth images are (strictly) downsampled to 2048$\times$2048.

\subsubsection{Video Experiments.}
For video experiments, we use the pretrained video SR model STAR~\cite{xie2025star} based on the I2VGen-XL backbone~\cite{2023i2vgenxl}, which is a widely used and publicly available baseline trained with text-conditioning.
Global and local (tiled) prompts are extracted using Qwen3-VL-8B-Instruct~\cite{qwen3-VL}.

We evaluate our framework on the real-world video datasets VideoLQ~\cite{chan2021investigating}, RealVSR~\cite{yang2021real}, and MVSR4x~\cite{wang2022benchmark}.
We do not apply spatial resizing or cropping for videos, and directly process entire video sequences as inputs.
For tiled inference, each LR latent is divided into spatio-temporal blocks of size $32 \times 720 \times 1280$.

\subsubsection{IQA Metrics.}
We assess the fidelity perceptual quality of images using no-reference metrics NIQE~\cite{zhang2015feature}, MUSIQ~\cite{ke2021musiq}, MANIQA~\cite{yang2022maniqa}, CLIPIQA~\cite{wang2023exploring}, and reference-based metrics PSNR, SSIM~\cite{wang2004image}, LPIPS~\cite{zhang2018unreasonable}, DISTS~\cite{ding2020image}, FID~\cite{heusel2017gans}.
For the VSR task, we evaluate frame-wise image quality on the metrics NIQE~\cite{zhang2015feature}, MUSIQ~\cite{ke2021musiq}, MANIQA~\cite{yang2022maniqa}, CLIPIQA~\cite{wang2023exploring}, HYPERIQA~\cite{su2020blindly}, and use video-specific metrics FasterVQA~\cite{wu2022neighbourhood}, FAST-VQA~\cite{wu2022fastvqa}, DOVER~\cite{wu2023exploring} to assess temporal consistency and authentic video quality.

\begin{figure}[!t]
    \centering
    \includegraphics[width=\linewidth]{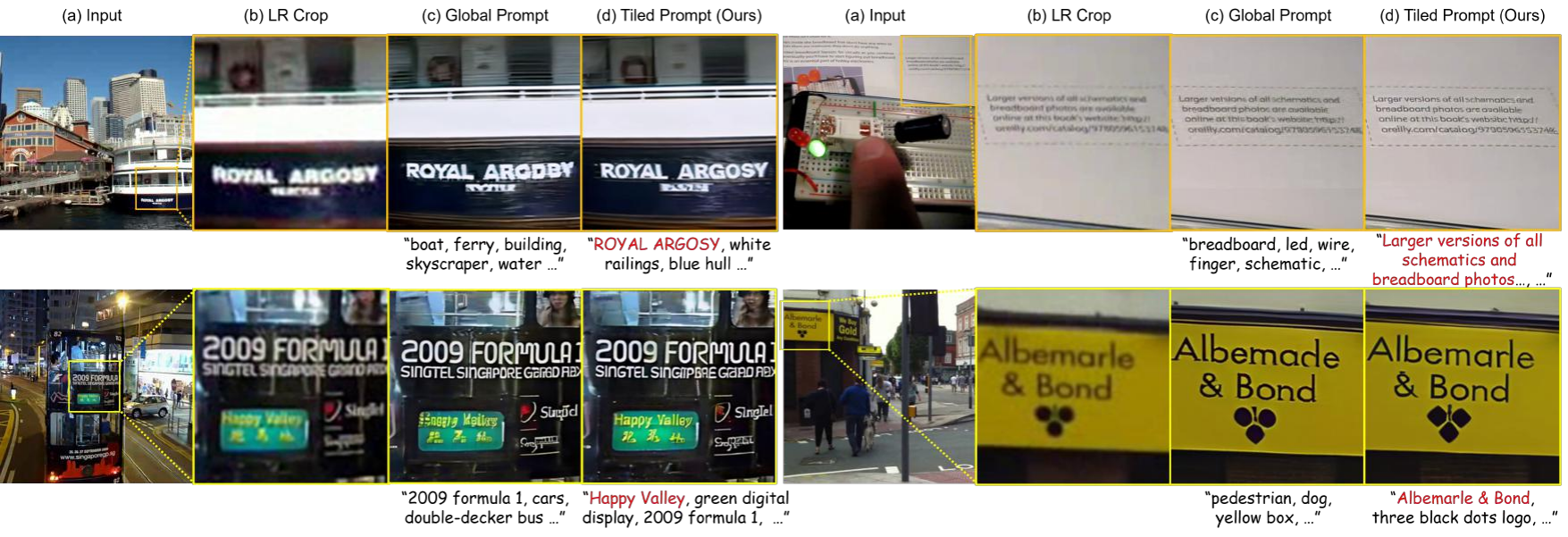}
    \caption{Qualitative results for video super-resolution. \textbf{(a,b) Input}: A low-resolution input frame and its cropped tile before SR. \textbf{(c) SR with Global Prompt}: Using only a coarse global prompt does not provide sufficient guidance, causing inaccurate SR results. \textbf{(d) SR with Tiled Prompt (Ours)}: Our method of using dense, localized textual guidance proves effectively reconstructs details in videos.}
    \label{fig:qualitative-video}
\end{figure}

\begin{table}[t]
    \centering
    \caption{Quantitative comparison on frame-wise quality, video quality, and video-text alignment metrics. Best results are in \textbf{bold}, second-best results are \underline{underlined}.}
    \resizebox{1.0\textwidth}{!}{
    \begin{tabular}{ll|ccccc |ccc |cc}
    \toprule
    \multicolumn{2}{l}{} & \multicolumn{5}{c}{\textbf{Frame-wise Quality}} & \multicolumn{3}{c}{\textbf{Video Quality}} & \multicolumn{2}{c}{\textbf{Video-Text Alignment}} \\
    & & \multirow{2}{*}{NIQE$\downarrow$} & \multirow{2}{*}{MUSIQ$\uparrow$} & \multirow{2}{*}{MANIQA$\uparrow$} & CLIP & HYPER & Faster & FAST & \multirow{2}{*}{DOVER$\uparrow$} & Language & VQAScore \\
    Dataset & Prompt Type & & & & -IQA$\uparrow$ & -IQA$\uparrow$ & VQA$\uparrow$ & -VQA$\uparrow$ & & Bind$\uparrow$ & (LLAVA)$\uparrow$ \\
    
    \midrule
    \multirow{4}{*}{VideoLQ}
    & \multicolumn{1}{c|}{\xmark}
    & 6.0567 & 36.5171 & 0.3860 & 0.2323 & 0.3341 & 0.6262 & 0.6368 & 48.4896 & \xmark & \xmark \\
    & Global (Baseline)
    & 4.8016 & 43.9160 & 0.4704 & 0.2648 & 0.3822 & 0.7354 & 0.7507 & 53.5704  & 0.1799 & 0.3987 \\
    & Global $+$ Local
    & \underline{4.7530} & \underline{44.4024} & \underline{0.4749} & \underline{0.2655} & \underline{0.3847} & \underline{0.7392} & \underline{0.7535} & \underline{53.9664} & \underline{0.1836} & \underline{0.4441} \\
    & \cg Local
    & \cg \textbf{4.5913} & \cg \textbf{45.7662} & \cg \textbf{0.4845} & \cg \textbf{0.2713} & \cg \textbf{0.3973} & \cg \textbf{0.7488} & \cg \textbf{0.7643} & \cg \textbf{54.4158} & \cg \textbf{0.1904} & \cg \textbf{0.5542} \\

    \midrule
    \multirow{4}{*}{RealVSR}
    & \multicolumn{1}{c|}{\xmark}
    & 4.7262 & 67.7858 & 0.6533 & 0.2792 & 0.5759 & \underline{0.7598} & \textbf{0.7383} & \underline{56.6598} & \xmark & \xmark \\
    & Global (Baseline)
    & \underline{4.0614} & \underline{69.9098} & 0.6738 & 0.3073 & 0.5790 & 0.7576 & 0.7268 & 55.1562 & \underline{0.2010} & 0.5673 \\
    & Global $+$ Local
    & \textbf{4.0495} & 69.8712 & \underline{0.6747} & \underline{0.3076} & \underline{0.5800} & \textbf{0.7604} & \underline{0.7277} & 55.3080 & 0.2008 & \underline{0.5860} \\
    & \cg Local
    & \cg 4.1152 & \cg \textbf{70.6071} & \cg \textbf{0.6817} & \cg \textbf{0.3213} & \cg \textbf{0.5865} & \cg 0.7560 & \cg 0.7244 & \cg \textbf{56.7844} & \cg \textbf{0.2043} & \cg \textbf{0.7002} \\

    \midrule
    \multirow{4}{*}{MVSR}
    & \multicolumn{1}{c|}{\xmark}
    & 6.0943 & 63.3108 & 0.5452 & 0.3396 & 0.5090 & 0.7345 & 0.7370 & \underline{57.6125} & \xmark & \xmark \\
    & Global (Baseline)
    & 4.9985 & \underline{67.8004} & 0.5836 & \textbf{0.3663} & 0.5532 & 0.7402 & 0.7336 & 53.1383 & 0.2127 & 0.5568 \\
    & Global $+$ Local
    & \underline{4.9882} & 67.6542 & \underline{0.5854} & 0.3599 & \underline{0.5550} & \underline{0.7877} & \underline{0.7744} & 57.2033 & \underline{0.2149} & \underline{0.6005} \\
    & \cg Local
    & \cg \textbf{4.8322} & \cg \textbf{68.4642} & \cg \textbf{0.5905} & \cg \underline{0.3612} & \cg \textbf{0.5598} & \cg \textbf{0.7925} & \cg \textbf{0.7808} & \cg \textbf{58.3464} & \cg \textbf{0.2168} & \cg \textbf{0.7456} \\
    
    \bottomrule
    \end{tabular}
    }
    \label{tab:quant-video}
\end{table}

\subsection{Comparison Results}

\subsubsection{Image/Video Quality Evaluation.}

We evaluate the performance of our framework regarding image restoration quality in Tab.~\ref{tab:quant-image} and video restoration quality in Tab.~\ref{tab:quant-video}.
Specifically, we compare four text-conditioning variants, (i) \textit{Null}: no text prompt; (ii) \textit{Global}: baseline scenario where a single prompt is shared across all tiles; (iii) \textit{Local}: tile-specific prompts extracted per $\tiledxl$; (iv) \textit{Global+Local}: concatenation of the global prompt to each tile-specific local prompt.

Quantitative results show that providing tiled (local) prompts as additional guidance consistently improves image and video quality compared to the baseline of using a single global prompt.
Notably, for VSR, using \textit{only} the local prompts yields better performance than using \textit{both} global and local prompts, even though the global$+$local variant contains all text tokens present in the local variant.
This suggests that the gains in restoration quality cannot be achieved by simply adding more text tokens; rather, the \textit{context} and relevance of the conditioning text are what drive the improvement.
The substantial gains in video quality metrics (\textit{i.e.}, FasterVQA, Fast-VQA, DOVER) also support that tiled prompts help alleviate \textit{both} spatial misguidance and temporal misguidance in VSR.

These results are further supported by qualitative comparison of image super-resolution provided in Fig.~\ref{fig:qualitative-image} and video super-resolution provided in Fig.~\ref{fig:qualitative-video}.
Our method of using tiled prompts provides dense, localized textual guidance to each local tile $\tiledxl$, facilitating reconstruction of local details like texture or words.

\begin{figure}[t]
\centering
\begin{minipage}[t]{0.64\linewidth}\vspace{0pt}
    \captionof{table}{Quantitative comparison on reference metrics. Best results are in \textbf{bold}.}
    \label{tab:quant-image-reference}
    \begin{center}
    \resizebox{\textwidth}{!}{
    \begin{tabular}{ll|ccccc}
    \toprule
    Dataset & Prompt Type & PSNR$\uparrow$ & SSIM$\uparrow$ & LPIPS$\downarrow$ & DISTS$\downarrow$ & FID$\downarrow$ \\
    
    \midrule
    \multirow{3}{*}{LSDIR2048}
    & Global (Baseline)
    & 22.55 & 0.5973 & 0.3000 & 0.1462 & 17.09 \\
    & Global $+$ Local
    & 22.54 & 0.5980 & 0.2967 & 0.1445 & 17.98 \\
    & \cg Local
    & \cg \textbf{22.63} & \cg \textbf{0.6011} & \cg \textbf{0.2900} & \cg \textbf{0.1400} & \cg \textbf{16.42} \\

    \midrule
    \multirow{3}{*}{SEPE8K}
    & Global (Baseline)
    & \textbf{23.00} & \textbf{0.6358} & 0.2852 & 0.1324 & 26.00 \\
    & Global $+$ Local
    & 22.97 & 0.6332 & 0.2844 & 0.1305 & 25.84 \\
    & \cg Local
    & \cg 22.96 & \cg 0.6335 & \cg \textbf{0.2809} & \cg \textbf{0.1276} & \cg \textbf{24.16} \\
    
    \bottomrule
    \end{tabular}
    }
    \end{center}
\end{minipage}\hfill
\begin{minipage}[t]{0.34\linewidth}\vspace{0pt}
    \captionof{table}{Inference time with latent tiling (seconds). The number of tiles are denoted inside parentheses.}
    \label{tab:runtime}
    \begin{center}
    \resizebox{0.75\textwidth}{!}{
    \begin{tabular}{lc}
        \toprule
        Method & Time \\
        \midrule
        DiT4SR & 162.58 \\
        $+$ Tiled Prompts (25) \quad & 166.15 \\
        \midrule
        STAR & 1273.9 \\
        $+$ Tiled Prompts (12) \quad & 1348.3 \\
        \bottomrule
    \end{tabular}
    }
    \end{center}
\end{minipage}
\end{figure}

\subsubsection{Prompt Quality Evaluation via Image/Video-Text Alignment.}
To validate that our framework alleviates prompt misguidance, we also measure image-text alignment between the input LR tiles $\tiledxl$ and the corresponding text prompts used for guidance (either $\vcglobal$ or $\vclocal$) in Tab~\ref{tab:quant-image}.
The video-text alignment results between the input spatio-temporal LR blocks $\tiledxl$ and the corresponding text prompts are provided in Tab~\ref{tab:quant-video}.
For image super-resolution we use the image-text alignment metrics CLIPScore~\cite{hessel2021clipscore}, ImageReward~\cite{xu2024imagereward}, HPSv2~\cite{wu2023human} to assess the relevance of the conditioning text to the local visual content, and for video super-resolution we use the metrics LanguageBind~\cite{zhu2024languagebind} and VQAScore\cite{lin2024evaluating} with LLaVA-OneVision~\cite{li2024llavaonevision}.

We compare across the three text-conditioning variants, excluding the null prompt case where alignment with text can not be measured. The baseline scenario uses a single global text prompt from a global extractor for all tiles, while our method of tiled prompts uses localized prompts for each tile.
Results show significant increase in alignment between local content and textual description when using local tile-specific prompts,
supporting that the local prompts indeed hold high relevance to each input LR tile.

\subsubsection{Reference-based Evaluation.}

As the main focus of our work is on mitigating prompt misguidance during tiled inference, we work on very high resolutions, making reference-based evaluation with high-resolution ground truth images difficult. Nonetheless, we provide quantitative comparison of reference-based metrics in Tab.~\ref{tab:quant-image-reference}, which verify the effectiveness of our method in improving both fidelity and perceptual quality.
Furthermore, the decrease in FID supports that images restored with tiled prompts better match the true image distribution.

\subsection{Further Discussion}

\subsubsection{Runtime Analysis.}

Tab.~\ref{tab:runtime} shows that the additional computation introduced by our framework incurs negligible overhead relative to the overall super-resolution inference cost. Consequently, the observed improvements in image and video quality justify the extra computation, making the trade-off favorable.

\subsubsection{Ablation over CFG Scales.}

We perform ablation over CFG scales $s$ to evaluate how performance changes as guidance strength increases. As shown in Fig.~\ref{fig:cfg-ablation}, our tiled prompts framework becomes progressively more effective as $s$ increases, yielding larger gains over the global prompt baseline. This trend suggests that stronger guidance amplifies prompt related issues for global prompts, while tiled prompts provide conditioning that remains robust and beneficial under high CFG.
Such results are consistent with our observation in \eqref{eqn:deltac_cfg} that the misguidance $s\deltac$ shifts the generative trajectory from the ideal direction. As CFG scale $s$ increases, the overall misguiding effects of $s\deltac$ are exacerbated, leading to higher discrepancy between using local prompts and global prompts.

\begin{figure}[!t]
    \centering
    \includegraphics[width=\linewidth]{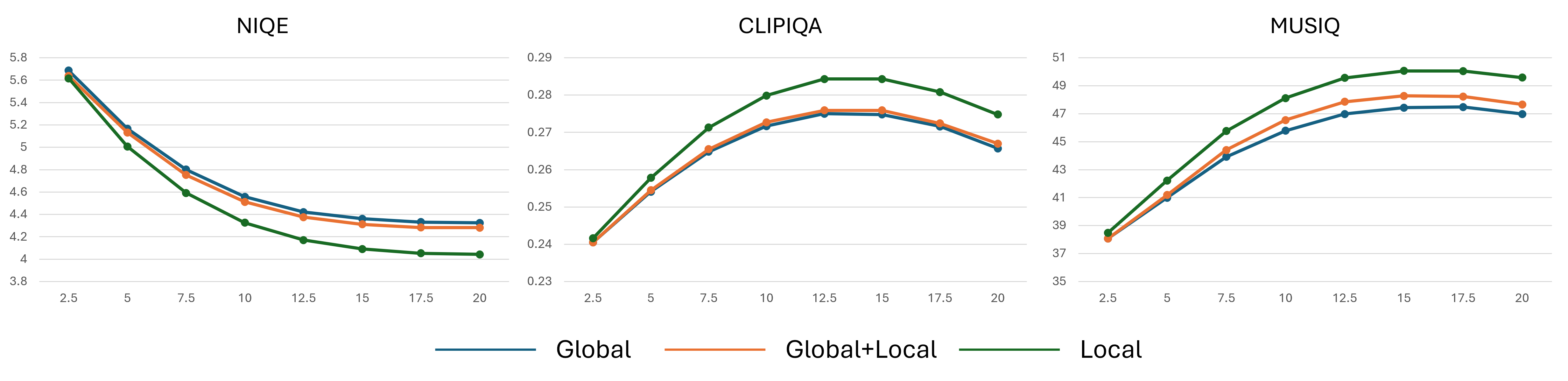}
    \caption{
    NIQE (lower is better), CLIPIQA and MUSIQ (higher is better) are reported across varying CFG scales, showing that tiled (local) prompts yield increasingly stronger improvements over the global prompt baseline as CFG increases.
    }
    \label{fig:cfg-ablation}
\end{figure}

\section{Conclusion}
\label{sec:conclusion}

We identified \textit{prompt misguidance} as a key failure mode of text-conditioned super-resolution under the latent tiling strategy, where a single global prompt is both too coarse to resolve tile-level ambiguity and prone to injecting irrelevant semantics. To address this, we proposed \textit{Tiled Prompts}, which replaces global text conditioning with tile-specific prompts to provide localized guidance for each latent tile. Across both SISR and VSR, tiled prompts improve perceptual quality and image/video-text alignment while introducing negligible computational overhead, offering a practical and unified solution for SR at high resolutions.

\subsubsection{Limitations.}
Our framework relies heavily on the performance of the prompt-extraction VLM, but VLMs can misinterpret local content or hallucinate details. Thus, an important direction for future work is to improve the robustness and reliability of VLM-based prompt extraction.

\clearpage

%
%
\bibliographystyle{splncs04}
\bibliography{main}

\clearpage
\appendix
\section{Proofs}
\label{sec-appendix:proof}

\lemmamisguidance*
\begin{proof}
    For an $\epsilon$-prediction model, the model output is 
    \begin{equation*}
        \epsilonb_\theta(\tiledxt, t) = -\sigma_t \nabla_\tiledxt\log{p_t(\tiledxt)}.
    \end{equation*}
    Substituting this into \eqref{eqn:deltac_def} directly yields $w(t) = \sigma_t$.\\
    
    For an $\vx_0$-prediction model, the model directly predicts the clean data $\tiledxzero$.\\
    From the forward process, 
    \begin{align*}
        \tiledxzero &= \frac{1}{\alpha_t}(\tiledxt - \sigma_t \epsilonb) \\
        &= \frac{1}{\alpha_t}\tiledxt + \frac{\sigma_t^2}{\alpha_t}\nabla_\tiledxt\log{p_t(\tiledxt)}.
    \end{align*}
    This results in
    \begin{align*}
         \deltac = \frac{\sigma_t^2}{\alpha_t} \left(\nabla_\tiledxt\log{p_t(\tiledxt |\tiledxl,\idealc)} - \nabla_\tiledxt\log{p_t(\tiledxt |  \tiledxl, \vc)} \right),
    \end{align*}
    and thus $w(t) = \frac{\sigma_t^2}{\alpha_t}$.\\
    
    For a $\vv$-prediction model, the target for prediction is $\vv_\theta(\tiledxt,t) = \alpha_t \epsilonb - \sigma_t \tiledxzero$. Substituting $\tiledxzero = \frac{1}{\alpha_t}(\tiledxt - \sigma_t \epsilonb)$ and $\epsilonb = -\sigma_t \nabla_\tiledxt\log{p_t(\tiledxt)}$ gives:
    \begin{equation}
        \vv_\theta(\tiledxt,t) = -\frac{\sigma_t}{\alpha_t}\left(\tiledxt + \nabla_\tiledxt\log{p_t(\tiledxt)}\right).
    \end{equation}
    By taking the difference between conditioning on $\vc$ and $\idealc$,
    \begin{align*}
         \deltac = \frac{\sigma_t}{\alpha_t} \left(\nabla_\tiledxt\log{p_t(\tiledxt |\tiledxl,\idealc)} - \nabla_\tiledxt\log{p_t(\tiledxt |  \tiledxl, \vc)} \right),
    \end{align*}
    resulting in $w(t) = \frac{\sigma_t}{\alpha_t}$.\\
    
    For linear flow-based models, $\alpha_t = 1-t$ and $\sigma_t = t$. The vector field is 
    \begin{equation}
        \vv_\theta(\tiledxt,t) = \mathbb{E}\left[\frac{\tiledxt-\tiledxzero}{t}\mid\tiledxt\right].
    \end{equation}
    By Tweedie's formula, the posterior expectation is 
    \begin{equation}
        \mathbb{E}[\tiledxzero\mid\tiledxt] = \frac{\tiledxt + t^2 \nabla_\tiledxt\log p_t(\tiledxt)}{1-t}.
    \end{equation}
    Substituting this into the vector field yields:
    \begin{equation}
        \vv_\theta(\tiledxt,t) = -\frac{1}{1-t}\left( \tiledxt + t\nabla_\tiledxt\log p_t(\tiledxt) \right).
    \end{equation}
    Therefore,
    \begin{equation}
        \deltac =\frac{t}{1-t}\left( \nabla_\tiledxt\log{p_t(\tiledxt |\tiledxl,\idealc)} - \nabla_\tiledxt\log{p_t(\tiledxt |  \tiledxl, \vc)} \right).
    \end{equation}
    Since $\alpha_t = 1-t$ and $\sigma_t = t$, this precisely matches $\frac{\sigma_t}{\alpha_t}$.
    \qed
\end{proof}

\propmutual*
\begin{proof}
    By the chain rule of conditional mutual information,
    \begin{align}
    I(\tiledxh;\idealc,\vc\mid \tiledxl)
    &= I(\tiledxh;\vc\mid \tiledxl) + I(\tiledxh;\idealc\mid \tiledxl,\vc), \label{eq:chain_a}\\
    I(\tiledxh;\idealc,\vc\mid \tiledxl)
    &= I(\tiledxh;\idealc\mid \tiledxl) + I(\tiledxh;\vc\mid \tiledxl,\idealc). \label{eq:chain_b}
    \end{align}
    Subtracting \eqref{eq:chain_a} from \eqref{eq:chain_b} yields
    \begin{equation}
    I(\tiledxh;\idealc\mid \tiledxl)-I(\tiledxh;\vc\mid \tiledxl)
    = I(\tiledxh;\idealc\mid \tiledxl,\vc) - I(\tiledxh;\vc\mid \tiledxl,\idealc).
    \end{equation}
    Using \eqref{eqn:cond_indep}, we have $I(\tiledxh;\vc\mid \tiledxl,\idealc)=0$, hence
    \begin{equation}
        \Delta I=I(\tiledxh;\idealc\mid \tiledxl,\vc).
    \end{equation}
    Finally, the conditional mutual information admits the standard KL form:
    \begin{align}
        I(\tiledxh;\idealc\mid \tiledxl,\vc)
        &=
        \mathbb{E}_{\idealc,\vc|\tiledxl}\!\left[
        D_{\mathrm{KL}}\!\Big(
        p(\tiledxh\mid \tiledxl,\idealc,\vc)\ \big\|\ p(\tiledxh\mid \tiledxl,\vc)
        \Big)\right] \nonumber\\
        &=
        \mathbb{E}_{\idealc,\vc|\tiledxl}\!\left[
        D_{\mathrm{KL}}\!\Big(
        p(\tiledxh\mid \tiledxl,\idealc)\ \big\|\ p(\tiledxh\mid \tiledxl,\vc)
        \Big)\right],
    \end{align}
    where the last equality uses \eqref{eqn:cond_indep}, ultimately proving \eqref{eqn:DeltaIKL}.

    Fix any realization $\idealc\sim p^*(\idealc|\tiledxl)$ and $\vc\sim p(\vc|\tiledxl)$.
    For any unified Gaussian marginal $p_t$, there exists an equivalent forward SDE:
    \begin{align*}
        d\tiledxt = f(\tiledxt, t)dt + g(t)d\vw_t.
    \end{align*}
    By Anderson's theorem~\cite{anderson1982reverse}, the corresponding reverse SDE can be written as
    \begin{align*}
        d\tiledxt = \{f(\tiledxt, t)-g(t)^2\nabla_\tiledxt \log p_t(\tiledxt)\}dt + g(t)d\vw_t.
    \end{align*}
    Thus, the reverse-time SDEs conditioned on $(\tiledxl,\idealc)$ and $(\tiledxl, \vc)$ share the same diffusion coefficient $g(t)$ but differ in their drift terms, driven by \\ 
    $\nabla_{\tiledxt}\log p_t(\tiledxt \mid \tiledxl, \idealc)$ and $\nabla_{\tiledxt}\log p_t(\tiledxt \mid \tiledxl, \vc)$, respectively.
    
    Let $P^{\idealc}$ and $P^{\vc}$ denote the corresponding reverse-time path measures, and let
    $p_0^{\idealc}$ and $p_0^{\vc}$ be their terminal marginals at $t=0$ (\textit{i.e.}, distributions of $\tiledxh$).
    By the data processing inequality,
    \begin{equation}
        D_{\mathrm{KL}}(p_0^{\idealc}\parallel p_0^{\vc})
        \le
        D_{\mathrm{KL}}(P^{\idealc}\parallel P^{\vc}).
        \label{eq:dpi}
    \end{equation}
    Using Girsanov's theorem and Lemma~\ref{lemma:misguidance}, the KL divergence between the two path measures satisfies
    \begin{align*}
        &D_{KL}(P^\idealc \parallel P^\vc) \\
        &= \frac{1}{2} \int_0^T \mathbb{E}_{P^\idealc} \left[ g(t)^2 \| \nabla_\tiledxt \log p_t(\tiledxt|\tiledxl,\idealc) - \nabla_\tiledxt \log p_t(\tiledxt|\tiledxl, \vc) \|^2 \right] dt\\
        &=\frac{1}{2} \int_0^T \frac{g(t)^2}{w(t)^2} \; \mathbb{E}_{P^*} \left[\| \deltac(t) \|^2\right]dt.
    \end{align*}
    Substituting this into \eqref{eq:dpi}, we have
    \begin{equation}
        D_{\mathrm{KL}}(p_0^{\idealc}\parallel p_0^{\vc})
        \le
        \frac{1}{2}\int_0^T \frac{g(t)^2}{w(t)^2}\;
        \mathbb{E}_{\tiledxt\sim p_t(\tiledxt\mid \tiledxl,\idealc)}\!\left[\|\deltac(t)\|^2\right]dt.
        \label{eq:fixed-draw-bound}
    \end{equation}
    Taking expectation of \eqref{eq:fixed-draw-bound} over $\idealc\sim p^*(\idealc|\tiledxl)$, $\vc\sim p(\vc|\tiledxl)$ and using linearity of expectation yields
    \begin{align*}
        \mathbb{E}_{(\idealc,\vc)\mid \tiledxl}\!&\left[D_{\mathrm{KL}}(p_0^{\idealc}\parallel p_0^{\vc})\right]\\
        \le
        &\frac{1}{2}\int_0^T \frac{g(t)^2}{w(t)^2}\;
        \mathbb{E}_{(\idealc,\vc)\mid \tiledxl}\;
        \mathbb{E}_{\tiledxt\sim p_t(\tiledxt\mid \tiledxl,\idealc)}\!\left[\|\deltac(t)\|^2\right]dt.
    \end{align*}
    Combining this with \eqref{eqn:DeltaIKL} (which identifies $\Delta I$ with the expected posterior KL) proves the desired lower bound of accumulated prompt misguidance.
    \qed
\end{proof}

\proptiled*

\begin{proof}
Fix an arbitrary realization $\{\vc_\text{local}^{(j)}\}_{j=1}^N$ and $\vcglobal$. 
Let us denote the set of tile-wise concepts $\mathcal{C}_{\mathrm{Local}}=\{\vc_\text{local}^{(1)},\dots,\vc_\text{local}^{(N)}\}$ which sufficiently specifies all tiles.
Assume the global prompt $\vcglobal$ is semantically incomplete and only spans a subset of these concepts. Let $S_g\subseteq\{1,\dots,N\}$ be the index set of concepts that are effectively represented in $\vcglobal$.
To model \emph{errors of omission}, we introduce an additional omitted condition $\vc_\emptyset$ representing the ambiguous fallback state when the relevant tile-specific concept is not available.

We introduce a latent variable $Z\in S_g\cup\{\emptyset\}$ indicating which semantic concept is effectively attended to when generating the $i$-th tile under $\vcglobal$.
Specifically, $Z=j$ means attending to $\vc_\text{local}^{(j)}$ (for $j\in S_g$), while $Z=\emptyset$ means the model falls back to $\vc_\emptyset$.
Then, for each fixed LR tile $\tiledxl$, we model the globally-guided conditional as the following mixture:
\begin{align}
    \label{eq:mixture-global}
    p(\tiledxh \mid \tiledxl, \vcglobal)
    = \sum_{z\in S_g\cup\{\emptyset\}} &w_{i,z}(\tiledxl;\vcglobal)\, p(\tiledxh \mid \tiledxl, \vc_z),\\
    \sum_{z} w_{i,z}(\tiledxl;\vcglobal)=1,\ \ &w_{i,z}(\tiledxl;\vcglobal)\ge 0,
\end{align}
where $\vc_z\!=\!\vc_\text{local}^{(z)}$ if $z\in S_g$ and $\vc_{\emptyset}$ if $z=\emptyset$.

Since entropy $H(p)=-p\log p$ is concave in the underlying distribution, applying Jensen's inequality to \eqref{eq:mixture-global} gives
\begin{equation}
    \label{eq:entropy-jensen}
    H(\tiledxh \mid \tiledxl,\vcglobal)
    \;\ge\;
    \sum_{z\in S_g\cup\{\emptyset\}} w_{i,z}(\tiledxl;\vcglobal)\, H(\tiledxh \mid \tiledxl,\vc_z).
\end{equation}

Assume the tile-specific concept $\vclocal$ is the most informative condition for tile $i$ in the sense that
\begin{align}
    \label{eq:entropy-min}
    H(\tiledxh \mid\tiledxl,\vclocal)
    \;&\le\;
    H(\tiledxh \mid \tiledxl,\vc_\text{local}^{(j)})\quad \forall j\neq i,\\
    H(\tiledxh \mid \tiledxl,\vclocal)
    \;&\le\;
    H(\tiledxh \mid \tiledxl,\vc_\emptyset). \label{eq:entropy-min_2}
\end{align}
The first inequality captures \emph{errors of commission} (irrelevant concepts $j\neq i$ increase ambiguity for tile $i$),
and the second captures \emph{errors of omission} (the fallback $\vc_\emptyset$ lacks high-frequency guidance and increases ambiguity).\\
Combining \eqref{eq:entropy-jensen}, \eqref{eq:entropy-min} and \eqref{eq:entropy-min_2} yields,
\begin{align}
    \label{eq:pointwise-entropy-order}
    H(\tiledxh \mid \tiledxl,\vcglobal)
    \;&\ge\;
    \sum_{z} w_{i,z}(\tiledxl;\vcglobal)\, H(\tiledxh \mid \tiledxl,\vclocal)\nonumber\\
    \;&=\;
    H(\tiledxh \mid \tiledxl,\vclocal). 
\end{align}
Since the realization $(\vcglobal,\{\vc_\text{local}^{(j)}\}_{j=1}^N)$ was arbitrary, \eqref{eq:pointwise-entropy-order} holds
\emph{almost surely} with respect to the randomness of the prompting scheme.

Taking expectation of \eqref{eq:pointwise-entropy-order} over the joint law of prompts conditioned on $\tiledxl$ gives
\begin{equation}
    \label{eq:entropy-order-expected}
    \mathbb{E}_{\vclocal, \vcglobal |\tiledxl}\big[\,H(\tiledxh \mid \tiledxl, \vcglobal) \big]
    \;\ge\;
    \mathbb{E}_{\vclocal, \vcglobal |\tiledxl}\big[\,H(\tiledxh \mid \tiledxl, \vc_\text{local}^{(i)}) \big].
\end{equation}

By definition,
$I(\tiledxh;\vc|\tiledxl)=H(\tiledxh|\tiledxl)-\mathbb{E}_{\vc|\tiledxl}[H(\tiledxh|\tiledxl,\vc)].$
Because $H(\tiledxh|\tiledxl)$ does not depend on the prompting scheme,
\eqref{eq:entropy-order-expected} implies
\begin{equation}
    I(\tiledxh;\vc_\text{local}^{(i)}\mid\tiledxl)\;\ge\; I(\tiledxh;\vcglobal\mid\tiledxl).
\end{equation}
Recalling $\Delta I_\ell= I(\tiledxh;\idealc|\tiledxl)-I(\tiledxh;\vclocal|\tiledxl)$ and $\Delta I_\textrm{g}= I(\tiledxh;\idealc|\tiledxl)-I(\tiledxh;\vcglobal|\tiledxl)$, where the first term
is identical for both schemes, we obtain
\begin{equation}
\Delta I_\ell\;\le\;\Delta I_\textrm{g},
\end{equation}
which concludes the proof.
\qed
\end{proof}

\section{Experimental Details}
\label{sec-appendix:exp-details}

\subsection{Model Checkpoints}
We use the pretrained VLM models Qwen2.5-VL-7B-Instruct and Qwen3-VL-8B-Instruct, available at \url{https://huggingface.co/Qwen/Qwen2.5-VL-7B-Instruct} and \url{https://huggingface.co/Qwen/Qwen3-VL-8B-Instruct}, respectively.
For the image super-resolution model, we use the pretrained checkpoint of DiT4SR available at \url{https://github.com/Adam-duan/DiT4SR}, specifically the `dit4sr\_q' checkpoint that was used in the original paper. For the video super-resolution model, we use the pretrained checkpoint of STAR at \url{https://github.com/NJU-PCALab/STAR}.

In Sec.~\ref{sec-appendix:quant} below, we provide additional quantitative comparison for SISR and VSR on various backbone models. Specifically, we use the pretrained checkpoint of OSEDiff available at \url{https://github.com/bryanswkim/Chain-of-Zoom}, the pretrained checkpoint of DiffVSR available at \url{https://github.com/xh9998/DiffVSR}, and the pretrained checkpoint of Upscale-A-Video available at \url{https://github.com/sczhou/Upscale-A-Video}.

\subsection{User Prompts}
\subsubsection{Image Super-Resolution.} The user prompt used by the VLM for generating both global and local prompts is as follows:
\begin{tcolorbox}[
    colback=black!5!white,
    colframe=black!75!white,
    fonttitle=\bfseries,
    breakable,
    arc=1mm,
    boxsep=1mm,
    left=2mm,
    right=2mm,
    top=2mm,
    bottom=2mm,
    fontupper=\scriptsize,
]
    \raggedright
    What is in this image? Give me a set of words.
\end{tcolorbox}

\subsubsection{Video Super-Resolution.} 
As stated in the main paper, one of our contributions is empirically setting up a prompt-extraction system that effectively extracts global and local prompts from a given video. Unlike for images, the video understanding capabilities of VLMs are yet substandard, requiring us to provide the VLM with more explicit and detailed instructions to generate highly descriptive keywords. First, to extract the global prompt $\vcglobal$ we provide the following instruction to the VLM:
\begin{tcolorbox}[
    colback=black!5!white,
    colframe=black!75!white,
    fonttitle=\bfseries,
    breakable,
    arc=1mm,
    boxsep=1mm,
    left=2mm,
    right=2mm,
    top=2mm,
    bottom=2mm,
    fontupper=\scriptsize,
]
    \raggedright
    Analyze this video with extreme focus and provide an exhaustive list of keywords.\\
    CRITICAL GUIDELINES:\\
    1. OCR \& Text Extraction: Transcribe EVERY visible character (digits, signage, logos, license plates, small labels). If it is readable, it MUST be in the list.\\
    2. Transient Objects: Identify objects that appear only for a split second or in the background(e.g., `raindrops'). Do not ignore fleeting details.\\
    3. Object Parts: Identify small components (e.g., `rivets', `bolts', `hinges', `seams', `cracks').\\
    4. Fine-grained Material Classification: Do not use vague words like `textured' or `smooth'.\\
    Instead, identify the EXACT material: (e.g., `anodized aluminum', `reinforced concrete', `tempered glass', `corrugated iron', `synthetic plastic', `terrazzo floor').\\
    Format: Output ONLY keywords separated by commas. No explanations, no `Here are the words', no `Note:'. NO generic adjectives like `nice', `clear', or `textured'.
\end{tcolorbox}

Then, we extract the tiled prompt $\vclocal$ for a specific local tile by using the following instruction:
\begin{tcolorbox}[
    colback=black!5!white,
    colframe=black!75!white,
    fonttitle=\bfseries,
    breakable,
    arc=1mm,
    boxsep=1mm,
    left=2mm,
    right=2mm,
    top=2mm,
    bottom=2mm,
    fontupper=\scriptsize,
]
    \raggedright
    The second video is a low-resolution crop (bicubic upsampled) of the first video. \\
    It may appear blurry due to upsampling, but you must ignore the blur.\\
    Compare both and describe the content of the SECOND video (the patch) with extreme detail:\\
    1. Intentional Texture: Based on the object's identity, what is its actual material?\\
    2. Micro-OCR: Transcribe any letters, numbers, or symbols that are unique to this patch.\\
    3. Edge \& Shape: Describe the intended sharp edges and structures of the objects in the patch.\\
    STRICT RULE: NEVER use words like `blurry', `pixelated', `noisy', `low-res', or `distorted'. \\
    Output ONLY the inferred high-quality keywords, separated by commas.
\end{tcolorbox}

\subsection{Other Settings.}

\subsubsection{Settings for Runtime Analysis.}
Runtime analysis of SISR and VSR tasks reported in Tab.~\ref{tab:runtime} was conducted with NVIDIA GeForce RTX 3090 GPUs.

\subsubsection{Settings for CFG Scale Ablation.}
The ablation over CFG scales reported in Fig.~\ref{fig:cfg-ablation} was performed using STAR~\cite{xie2025star} as the backbone architecture. Eight CFG scales $s \in \{2.5, 5.0, 7.5, 10.0, 12.5, 15.0, 17.5, 20.0\}$ were evaluated to thoroughly observe the trend of prompt misguidance under varying guidance strengths.

\section{Additional Quantitative Results}
\label{sec-appendix:quant}

We verify the model-agnostic behavior of our proposed framework by providing additional quantitative results on various pretrained super-resolution backbones, for both image and video super-resolution.

\subsubsection{Image Super-Resolution.}
We provide additional quantitative comparison for SISR using the OSEDiff~\cite{wu2024one} model in Tab.~\ref{tab:quant-image-osediff}. As the OSEDiff model is trained on the LSDIR~\cite{li2023lsdir} and FFHQ~\cite{karras2019style} datasets, evaluation is performed on the training datasets of DIV2K~\cite{agustsson2017ntire} and DIV8K~\cite{gu2019div8k}, consisting of 800 images and 1500 images respectively.
Results show that using local (tiled) prompts significantly improve performance regarding image quality and image-text alignment over the global prompt baseline.

\subsubsection{Video Super-Resolution.}
We provide additional quantitative comparison for VSR using the DiffVSR~\cite{li2025diffvsr} model in Tab.~\ref{tab:quant-video-diffvsr} and the Upscale-A-Video~\cite{zhou2023upscale} model in Tab.~\ref{tab:quant-video-uav}. As for the STAR~\cite{xie2025star} model, the VideoLQ~\cite{chan2021investigating}, RealVSR~\cite{yang2021real}, and MVSR~\cite{wang2022benchmark} datasets are used for evaluation.
Results show that using local (tiled) prompts significantly improve performance regarding frame-wise quality, video quality, and video-text alignment over the global prompt baseline.

\section{Additional Qualitative Results}
We provide additional qualitative comparisons in Figs.~\ref{fig:appendix-image}-\ref{fig:appendix-video}.
%

\begin{table}[t]
    \centering
    \caption{Quantitative comparison for SISR on image quality metrics and image-text alignment metrics experimented on the OSEDiff model. Best results are in \textbf{bold}, second-best results are \underline{underlined}.}
    \resizebox{0.8\textwidth}{!}{
    \begin{tabular}{ll|cccc|ccc}
    \toprule
    \multicolumn{2}{l}{} & \multicolumn{4}{c}{\textbf{Image Quality}} & \multicolumn{3}{c}{\textbf{Image-Text Alignment}} \\
    Dataset & Prompt Type & NIQE$\downarrow$ & MUSIQ$\uparrow$ & MANIQA$\uparrow$ & CLIPIQA$\uparrow$ & CLIP Score$\uparrow$ & ImageReward$\uparrow$ & HPSv2$\uparrow$ \\
    
    \midrule
    \multirow{4}{*}{DIV2K}
    & \multicolumn{1}{c|}{\xmark}
    & 3.7594 & 64.8993 & 0.5797 & 0.6586 & \xmark & \xmark & \xmark \\
    & Global (Baseline)
    & 3.7234 & 65.0033 & 0.5829 & 0.6578 & \underline{26.6749} & -1.4630 & 0.1685 \\
    & Global $+$ Local
    & \underline{3.6993} & \textbf{65.3359} & \textbf{0.5856} & \underline{0.6643} & 26.0289 & \underline{-1.3694} & \underline{0.1738} \\
    & \cg Local
    & \cg \textbf{3.6821} & \cg \underline{65.2489} & \cg \underline{0.5835} & \cg \textbf{0.6649} & \cg \textbf{27.1750} & \cg \textbf{-0.6806} & \cg \textbf{0.1998} \\

    \midrule
    \multirow{4}{*}{DIV8K}
    & \multicolumn{1}{c|}{\xmark}
    & 3.8123 & 64.6939 & 0.5819 & 0.6578 & \xmark & \xmark & \xmark \\
    & Global (Baseline)
    & 3.8098 & 65.0295 & \underline{0.5875} & 0.6608 & \underline{26.1589} & -1.6394 & 0.1603 \\
    & Global $+$ Local
    & \underline{3.7788} & \textbf{65.2291} & \textbf{0.5896} & \textbf{0.6650} & 25.6802 & \underline{-1.5527} & \underline{0.1664} \\
    & \cg Local
    & \cg \textbf{3.7597} & \cg \underline{65.1000} & \cg 0.5862 & \cg \underline{0.6646} & \cg \textbf{27.2356} & \cg \textbf{-0.7073} & \cg \textbf{0.2015} \\
    
    \bottomrule
    \end{tabular}
    }
    \label{tab:quant-image-osediff}
\end{table}

\begin{table}[t]
    \centering
    \caption{Quantitative comparison for VSR on frame-wise quality, video quality, and video-text alignment metrics experimented on the DiffVSR model. Best results are in \textbf{bold}, second-best results are \underline{underlined}.}
    \resizebox{1.0\textwidth}{!}{
    \begin{tabular}{ll|ccccc |ccc |cc}
    \toprule
    \multicolumn{2}{l}{} & \multicolumn{5}{c}{\textbf{Frame-wise Quality}} & \multicolumn{3}{c}{\textbf{Video Quality}} & \multicolumn{2}{c}{\textbf{Video-Text Alignment}} \\
    & & \multirow{2}{*}{NIQE$\downarrow$} & \multirow{2}{*}{MUSIQ$\uparrow$} & \multirow{2}{*}{MANIQA$\uparrow$} & CLIP & HYPER & Faster & FAST & \multirow{2}{*}{DOVER$\uparrow$} & Language & VQAScore \\
    Dataset & Prompt Type & & & & -IQA$\uparrow$ & -IQA$\uparrow$ & VQA$\uparrow$ & -VQA$\uparrow$ & & Bind$\uparrow$ & (LLAVA)$\uparrow$ \\
    
    \midrule
    \multirow{4}{*}{VideoLQ}
    & \multicolumn{1}{c|}{\xmark}
    & 4.2160 & 46.5694 & 0.4675 & 0.2851 & 0.3949 & 0.7404 & 0.7271 & 47.8760 & \xmark & \xmark \\
    & Global (Baseline)
    & 3.6857 & 57.4907 & 0.5362 & 0.3294 & 0.4841 & 0.7757 & 0.7576 & \underline{51.2854}  & 0.1752 & 0.3109 \\
    & Global $+$ Local
    & \underline{3.6321} & \underline{58.2223} & \underline{0.5408} & \underline{0.3315} & \underline{0.4903} & \underline{0.7765} & \underline{0.7589} & 51.1752 & \underline{0.1788} & \underline{0.3745} \\
    & \cg Local
    & \cg \textbf{3.5742} & \cg \textbf{58.7907} & \cg \textbf{0.5442} & \cg \textbf{0.3380} & \cg \textbf{0.4959} & \cg \textbf{0.7796} & \cg \textbf{0.7654} & \cg \textbf{51.5142} & \cg \textbf{0.1826} & \cg \textbf{0.4643} \\

    \midrule
    \multirow{4}{*}{RealVSR}
    & \multicolumn{1}{c|}{\xmark}
    & \textbf{3.5370} & 70.3216 & 0.6512 & 0.2864 & 0.5627 & 0.7802 & 0.7625 & 50.2731 & \xmark & \xmark \\
    & Global (Baseline)
    & 3.7112 & 73.4731 & \textbf{0.6790} & 0.3031 & 0.6240 & \underline{0.7858} & \textbf{0.7736} & 52.5130 & \underline{0.2012} & 0.5542 \\
    & Global $+$ Local
    & \underline{3.6647} & \underline{73.7893} & \underline{0.6781} & \underline{0.3061} & \underline{0.6306} & 0.7848 & 0.7674 & \underline{52.5698} & 0.2007 & \underline{0.5587} \\
    & \cg Local
    & \cg 3.6802 & \cg \textbf{73.9752} & \cg 0.6771 & \cg \textbf{0.3101} & \cg \textbf{0.6365} & \cg \textbf{0.7862} & \cg \underline{0.7686} & \cg \textbf{52.5852} & \cg \textbf{0.2061} & \cg \textbf{0.7079} \\

    \midrule
    \multirow{4}{*}{MVSR}
    & \multicolumn{1}{c|}{\xmark}
    & 4.9609 & 65.7228 & 0.5774 & 0.3705 & 0.5353 & 0.7874 & 0.7615 & 51.1141 & \xmark & \xmark \\
    & Global (Baseline)
    & \underline{4.5918} & 71.5383 & \underline{0.6071} & \underline{0.4087} & 0.5977 & 0.7968 & \underline{0.7798} & \underline{54.2107} & 0.2079 & 0.5309 \\
    & Global $+$ Local
    & \textbf{4.5717} & \underline{71.5397} & 0.6070 & 0.4060 & \underline{0.5991} & \underline{0.8019} & \textbf{0.7804} & 54.2080 & \underline{0.2094} & \underline{0.5777} \\
    & \cg Local
    & \cg 4.7488 & \cg \textbf{71.8988} & \cg \textbf{0.6156} & \cg \textbf{0.4220} & \cg \textbf{0.6110} & \cg \textbf{0.8054} & \cg 0.7770 & \cg \textbf{54.7193} & \cg \textbf{0.2107} & \cg \textbf{0.7023} \\
    
    \bottomrule
    \end{tabular}
    }
    \label{tab:quant-video-diffvsr}
\end{table}

\begin{table}[t]
    \centering
    \caption{Quantitative comparison for VSR on frame-wise quality, video quality, and video-text alignment metrics experimented on the Upscale-A-Video model. Best results are in \textbf{bold}, second-best results are \underline{underlined}.}
    \resizebox{1.0\textwidth}{!}{
    \begin{tabular}{ll|ccccc |ccc |cc}
    \toprule
    \multicolumn{2}{l}{} & \multicolumn{5}{c}{\textbf{Frame-wise Quality}} & \multicolumn{3}{c}{\textbf{Video Quality}} & \multicolumn{2}{c}{\textbf{Video-Text Alignment}} \\
    & & \multirow{2}{*}{NIQE$\downarrow$} & \multirow{2}{*}{MUSIQ$\uparrow$} & \multirow{2}{*}{MANIQA$\uparrow$} & CLIP & HYPER & Faster & FAST & \multirow{2}{*}{DOVER$\uparrow$} & Language & VQAScore \\
    Dataset & Prompt Type & & & & -IQA$\uparrow$ & -IQA$\uparrow$ & VQA$\uparrow$ & -VQA$\uparrow$ & & Bind$\uparrow$ & (LLAVA)$\uparrow$ \\
    
    \midrule
    \multirow{4}{*}{VideoLQ}
    & \multicolumn{1}{c|}{\xmark}
    & 5.8585 & 28.3939 & 0.3407 & 0.2101 & 0.2958 & 0.4628 & 0.4105 & 40.5028 & \xmark & \xmark \\
    & Global (Baseline)
    & 3.7983 & 51.9114 & 0.4917 & 0.3903 & 0.4527 & 0.7295 & 0.7006 & 49.9849  & 0.1789 & 0.3109 \\
    & Global $+$ Local
    & \underline{3.7511} & \underline{52.4725} & \underline{0.4966} & \underline{0.3936} & \underline{0.4556} & \underline{0.7313} & \underline{0.7074} & \underline{50.3215} & \underline{0.1808} & \underline{0.3404} \\
    & \cg Local
    & \cg \textbf{3.5909} & \cg \textbf{54.9391} & \cg \textbf{0.5123} & \cg \textbf{0.4230} & \cg \textbf{0.4774} & \cg \textbf{0.7593} & \cg \textbf{0.7298} & \cg \textbf{51.7112} & \cg \textbf{0.1868} & \cg \textbf{0.4842} \\

    \midrule
    \multirow{4}{*}{RealVSR}
    & \multicolumn{1}{c|}{\xmark}
    & \textbf{3.0595} & 63.0827 & 0.5755 & 0.3314 & 0.4584 & 0.7024 & \textbf{0.6529} & 40.5028 & \xmark & \xmark \\
    & Global (Baseline)
    & 3.4213 & \underline{71.2438} & \underline{0.6168} & \underline{0.4911} & \underline{0.6270} & \underline{0.7052} & \underline{0.6215} & \underline{44.0409} & \underline{0.2012} & \underline{0.5761} \\
    & Global $+$ Local
    & \underline{3.4205} & 71.2431 & 0.6167 & 0.4910 & \underline{0.6270} & \textbf{0.7060} & 0.6209 & 44.0249 & 0.2009 & 0.5585 \\
    & \cg Local
    & \cg 3.7017 & \cg \textbf{71.8502} & \cg \textbf{0.6205} & \cg \textbf{0.5252} & \cg \textbf{0.6453} & \cg 0.6931 & \cg 0.5973 & \cg \textbf{44.6735} & \cg \textbf{0.2051} & \cg \textbf{0.7256} \\

    \midrule
    \multirow{4}{*}{MVSR}
    & \multicolumn{1}{c|}{\xmark}
    & 5.7918 & 53.8781 & 0.4871 & 0.3550 & 0.4299 & 0.6148 & 0.5917 & 42.0209 & \xmark & \xmark \\
    & Global (Baseline)
    & \textbf{4.3328} & 71.1735 & \underline{0.5853} & 0.6186 & 0.6232 & \textbf{0.7812} & \underline{0.7319} & 52.0487 & 0.2082 & 0.5449 \\
    & Global $+$ Local
    & \underline{4.3665} & \underline{71.2803} & \underline{0.5853} & \underline{0.6242} & \underline{0.6257} & \underline{0.7731} & \textbf{0.7325} & \underline{52.0815} & \underline{0.2093} & \underline{0.5715} \\
    & \cg Local
    & \cg 4.6500 & \cg \textbf{71.9011} & \cg \textbf{0.5951} & \cg \textbf{0.6552} & \cg \textbf{0.6497} & \cg 0.7596 & \cg 0.6982 & \cg \textbf{53.2388} & \cg \textbf{0.2126} & \cg \textbf{0.7217} \\
    
    \bottomrule
    \end{tabular}
    }
    \label{tab:quant-video-uav}
\end{table}

\begin{figure}[h]
    \centering
    \includegraphics[width=\linewidth]{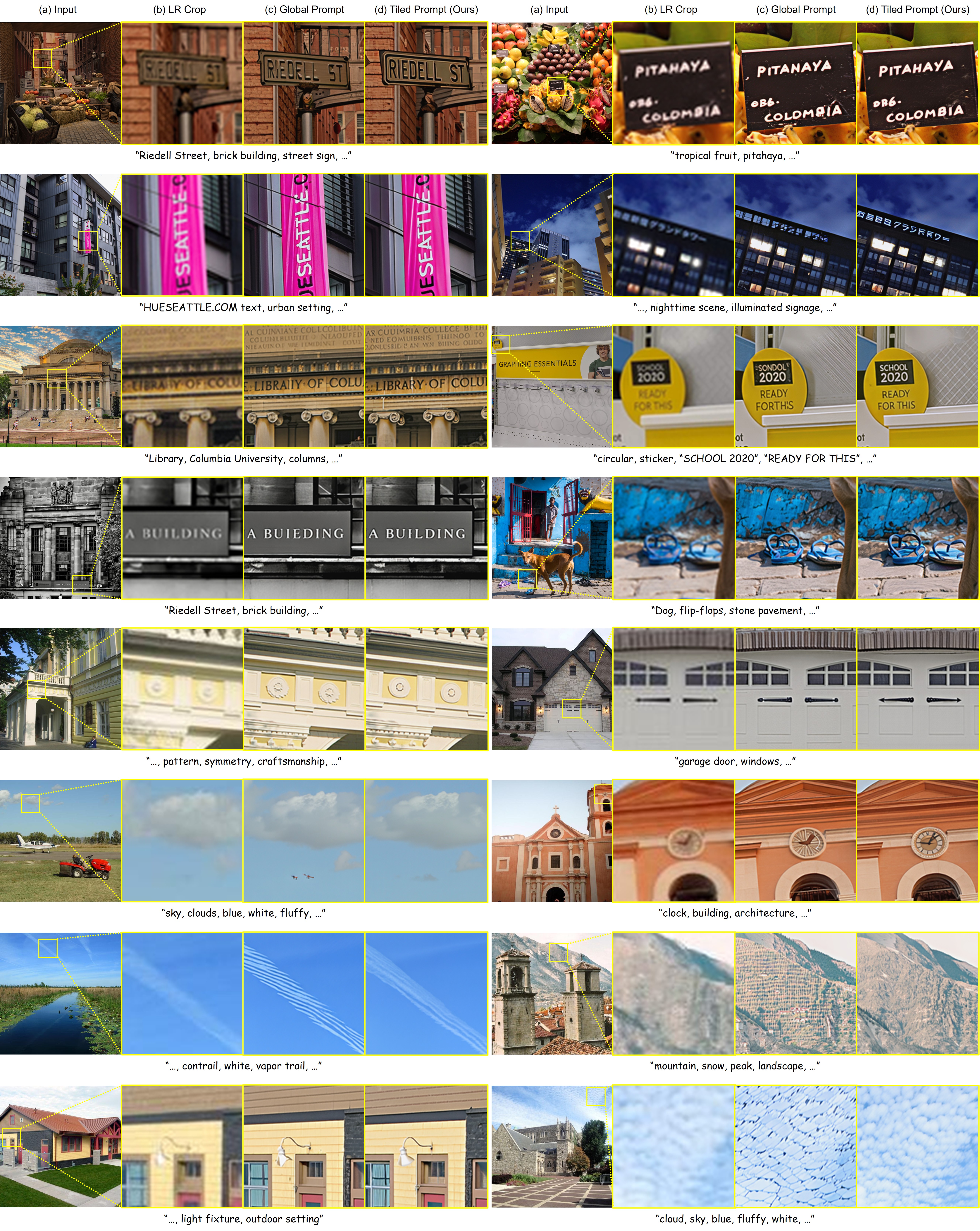}
    \caption{Additional qualitative results on the SISR model DiT4SR comparing \textbf{(a,b) the input image and its low-resolution crop}, \textbf{(c) SR results using the global prompt}, and \textbf{(d) SR results using the tiled prompt}. The text prompt below depicts a relevant part of the tiled prompt used to mitigate prompt misguidance and effectively guide the super-resolution process. }
    \label{fig:appendix-image}
\end{figure}

\begin{figure}[h]
    \centering
    \includegraphics[width=\linewidth]{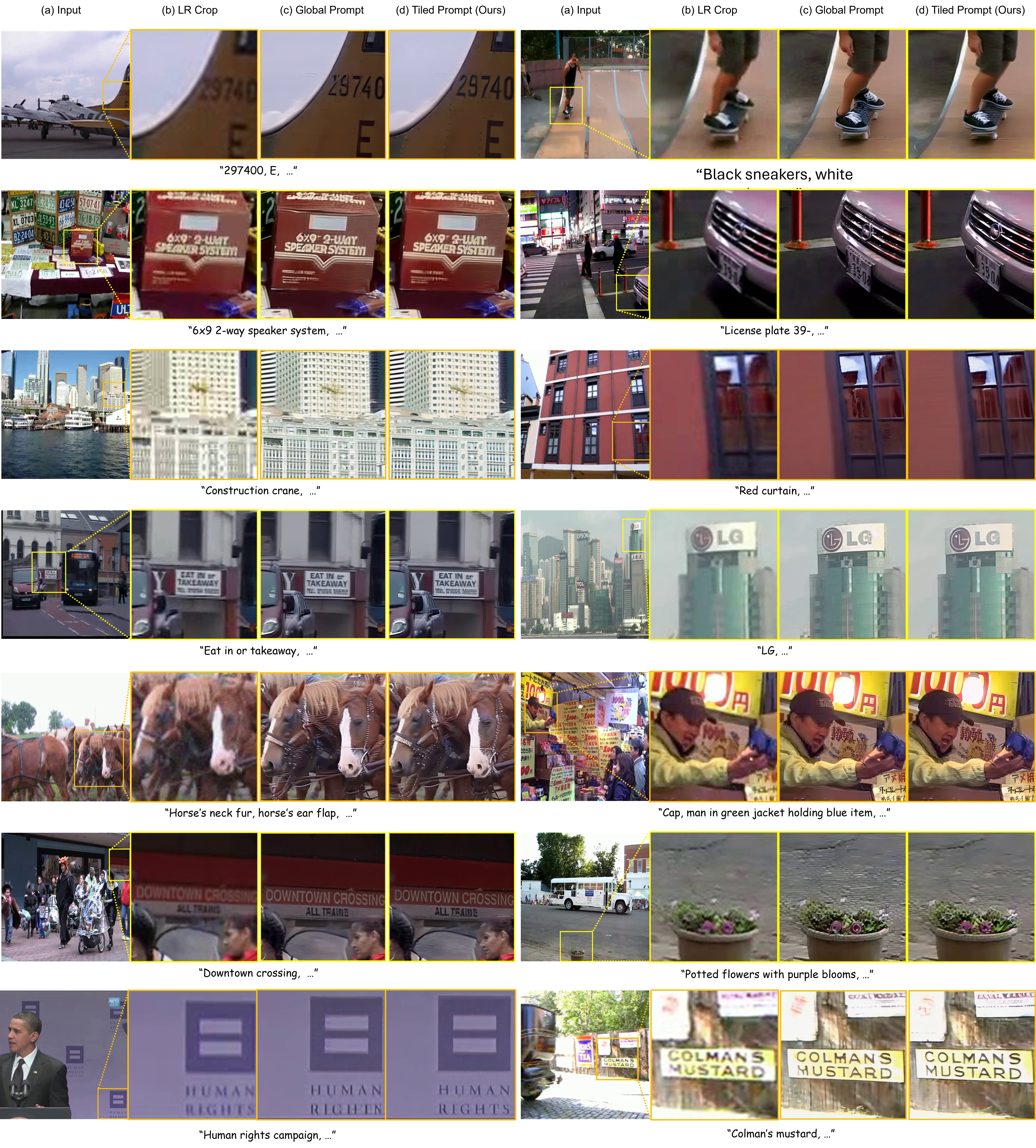}
    \caption{Additional qualitative results on the VSR model STAR comparing \textbf{(a,b) the input image and its low-resolution crop}, \textbf{(c) SR results using the global prompt}, and \textbf{(d) SR results using the tiled prompt}. The text prompt below depicts a relevant part of the tiled prompt used to mitigate prompt misguidance and effectively guide the super-resolution process. }
    \label{fig:appendix-video}
\end{figure}

\end{document}